\documentclass[10pt,journal,compsoc]{IEEEtran}
\ifCLASSOPTIONcompsoc
  \usepackage[nocompress]{cite}
\else
  \usepackage{cite}
\fi
\ifCLASSINFOpdf
\else
\fi
\usepackage[utf8]{inputenc} 
\usepackage[T1]{fontenc}
\usepackage{hyperref}       

\usepackage{url}            
\usepackage{booktabs}       
\usepackage{amsfonts}       
\usepackage{nicefrac}       
\usepackage{microtype}      
\usepackage{color}

\usepackage{hyperref}       

\usepackage{tabulary}
\usepackage{colortbl}
\usepackage{float}
\usepackage{tabularx}
\usepackage{fontawesome}

\usepackage[table,xcdraw]{xcolor}
\newcommand\crule[3][black]{\textcolor{#1}{\rule{#2}{#3}}}

\definecolor{darkblue}{rgb}{0,0.08,0.45}

\usepackage{wrapfig}
\usepackage{amsmath}
\usepackage{amsfonts}
\usepackage{import}
\usepackage{amsthm}
\usepackage{soul}
\usepackage{amssymb}
\usepackage{graphicx}
\usepackage{color}
\usepackage{rotating}
\usepackage{multirow}
\usepackage{subfiles}

\usepackage{float}

\DeclareMathOperator{\diag}{diag}
\DeclareMathOperator{\argmin}{argmin}
\DeclareMathOperator{\argmax}{argmax}

\DeclareMathOperator{\median}{median}
\newcommand{\uargmin}[1]{\underset{#1}{\argmin}\;}
\newcommand{\uargmax}[1]{\underset{#1}{\argmax}\;}
\newtheorem{prop}{Proposition}

\newcommand{\Un}{\mathbf{1}}

\newcommand{\norm}[1]{|\!| #1 |\!|}
\newcommand{\normu}[1]{\norm{#1}_{1}}

\newcommand{\normd}[1]{\norm{#1}_{2}}

\newcommand{\valL}{L}
\newcommand{\subgradp}{\phi}
\newcommand{\indexL}{l}
\newcommand{\const}{K}

\def\R{{\mathbb R}}

\usepackage{paralist}

\definecolor{columbiablue}{rgb}{0.61, 0.87, 1.0}
\definecolor{babyblueeyes}{rgb}{0.63, 0.79, 0.95}
\definecolor{beaublue}{rgb}{0.74, 0.83, 0.9} 
\definecolor{darkcandyapplered}{rgb}{0.64, 0.0, 0.0}
\definecolor{amaranth}{rgb}{0.9, 0.17, 0.31}
\definecolor{lightcoral}{rgb}{0.94, 0.5, 0.6}
\definecolor{piggypink}{rgb}{0.98, 0.75, 0.9}
\definecolor{tomato}{rgb}{1.0, 0.65, 0.65}
\definecolor{magicmint}{rgb}{0.67, 0.94, 0.82}
\definecolor{carnationpink}{rgb}{1.0, 0.65, 0.89}
\definecolor{palepink}{rgb}{0.98, 0.85, 0.87}
\definecolor{dukeblue}{rgb}{0.0, 0.0, 0.61}
\usepackage{tikz}

\definecolor{colorbrewer0}{RGB}{45,45,45}
\definecolor{colorbrewer1}{RGB}{228,26,28}
\definecolor{colorbrewer2}{RGB}{55,126,184}
\definecolor{colorbrewer3}{RGB}{77,175,74}
\definecolor{colorbrewer4}{RGB}{152,78,163}
\definecolor{colorbrewer5}{RGB}{255,127,0}
\definecolor{colorbrewer6}{RGB}{255,255,51}
\definecolor{colorbrewer7}{RGB}{166,86,40}
\definecolor{colorbrewer8}{RGB}{247,129,191}
\definecolor{colorbrewer9}{RGB}{153,153,153}
\definecolor{colorbrewer10}{RGB}{24,167,181}
\definecolor{pinegreen}{rgb}{0.0, 0.47, 0.44}

\definecolor{navyblue}{rgb}{0.0, 0.0, 0.5}
\definecolor{green2}{HTML}{e6f5f0}

\definecolor{bleudefrance}{rgb}{0.19, 0.55, 0.91}
\definecolor{citecolor}{HTML}{0071bc}

\hypersetup{
  colorlinks   = true, 
  linkcolor    = blue, 
  citecolor   = blue 
}

\newcommand{\AngieItem}[1]{\tikz[baseline=(AngieItem.base),remember
picture]{%
\node[fill=columbiablue,inner sep=2pt] (AngieItem){#1};}}
\newcommand{\AngieHighlight}{\tikz[overlay,remember picture]{%
\fill[black] ([yshift=2pt,xshift=-\pgflinewidth]AngieItem.east) -- ++(1.5pt,-1.5pt)
-- ++(-1.5pt,-1.5pt) -- cycle;
}}
\newcommand{\AngieItemRed}[1]{\tikz[baseline=(AngieItemRed.base),remember
picture]{%
\node[fill=magicmint,inner sep=2pt] (AngieItemRed){#1};}}
\newcommand{\AngieHighlightB}{\tikz[overlay,remember picture]{%
\fill[black] ([yshift=2pt,xshift=-\pgflinewidth]AngieItemRed.east) -- ++(1.5pt,-1.5pt)
-- ++(-1.5pt,-1.5pt) -- cycle;
}}

\begin{document}
%
\title{Energy Models for Better Pseudo-Labels:
Improving Semi-Supervised Classification with the 1-Laplacian Graph Energy}

\author{Angelica I~Aviles-Rivero,
        Nicolas~Papadakis,
        Ruoteng~Li,
        Philip Sellars,
        Samar M~Alsaleh, \\
        Robby T~Tan
        and~Carola-Bibiane~Sch\"{o}nlieb
\IEEEcompsocitemizethanks{\IEEEcompsocthanksitem AI Aviles-Rivero, P sellars and CB Sch\"{o}nlieb are with Department of Applied Mathematics and Theoretical Physics, University of Cambridge, UK \protect
\{ai323,ps644,cbs31\}@cam.ac.uk
\IEEEcompsocthanksitem N Papadakis is with the IMB, Université Bordeaux, France.  \{nicolas.papadakis@math.u-bordeaux.fr\}
\IEEEcompsocthanksitem R Li is with ByteDance, Singapore. rhein@bytedance.com
\IEEEcompsocthanksitem S Alsaleh is with the Department of Computer Science, Taibah University, KSA. asamar@taibahu.edu.sa
\IEEEcompsocthanksitem RT Tan is with National University of Singapore and Yale-NUS College, Singapore. robby.tan@nus.edu.sg.
}
}

\IEEEtitleabstractindextext{%
\begin{abstract}
Semi-supervised classification is a great focus of interest, as in real-world scenarios obtaining labels is expensive, time-consuming and might require expert knowledge. This has motivated the fast development of semi-supervised techniques, whose  performance is on a par with or better than supervised approaches.
A current major challenge for semi-supervised techniques is how to better handle the network calibration and confirmation bias problems for improving performance. In this work, we argue that  energy models are an effective alternative to such problems. With this motivation in mind, we propose a hybrid framework for semi-supervised classification called CREPE model  (1-Lapla\textbf{C}ian g\textbf{R}aph \textbf{E}nergy for \textbf{P}seudo-lab\textbf{E}ls). Firstly, we introduce a new energy model based on the non-smooth $\ell_1$ norm of the normalised graph 1-Laplacian. Our functional enforces a sufficiently smooth solution and  strengthens the intrinsic relation between the labelled and unlabelled data. Secondly, we provide a theoretical analysis for our proposed scheme and show that the solution trajectory does converge to a non-constant steady point. Thirdly, we derive the connection of our energy model for pseudo-labelling. We show that our energy model produces more meaningful pseudo-labels than the ones generated directly by a deep network. We extensively evaluate our framework, through numerical and visual experiments, using six benchmarking datasets for natural and medical images. We demonstrate that our technique reports state-of-the-art results for semi-supervised classification.
\end{abstract}

\begin{IEEEkeywords}
Semi-Supervised Learning, Energy Models, Graph Laplacian, Pseudo-Labelling, Image Classification,  Deep Learning
\end{IEEEkeywords}}

\maketitle

\IEEEdisplaynontitleabstractindextext

\IEEEpeerreviewmaketitle

\IEEEraisesectionheading{\section{Introduction}\label{sec:introduction}}
\IEEEPARstart{I}n this era of big data, deep learning (DL) has reported astonishing results for different tasks in computer vision.
For the task of image classification, a major breakthrough has been reported in the setting of supervised learning. In this context, the majority of methods are based on deep convolutional neural networks e.g.~\cite{simonyan2014very,he2016deep}, in which pre-trained, fine-tuned and trained from scratch solutions have been considered. A key factor for these impressive results is the assumption of a large and well-representative corpus of labelled data. These labels can be generated either by humans or automatically on proxy tasks.
However, obtaining well-annotated labels is expensive, time consuming and one should account for the inherent human bias and uncertainty that adversely effect the classification output. These drawbacks have motivated semi-supervised learning (SSL)~\cite{chapelle2006semi,zhu2009introduction} to be a great focus of interest for the community.

The key idea of SSL is to leverage on a tiny labelled set and a large unlabelled set to produce a good classification output. The desirable advantages of this setting is that one can decrease the dependency on a large amount of well-annotated data whilst gaining further understanding of intrinsic data structures~\cite{chapelle2006semi}.
The body of literature has reported promising results, from the classic perspective, for semi-supervised classification using both  transductive (e.g~\cite{zhu2002learning,zhou2004learning,wang2008graph,zhu2002learning5,joachims2003transductive,hein2010inverse,zhang2011fast}) and inductive (e.g.~\cite{zhu2005harmonic,delalleau2005efficient,belkin2006manifold}) philosophies. Those techniques seek to infer the labels for the large unlabelled set, relying solely on the tiny labelled set as prior, by minimising a given energy (i.e., energy models~\cite{chapelle2006semi,zhu2002learning,zhou2004learning}).  More recently, deep learning has also been applied in the SSL context - examples are~\cite{rasmus2015semi,laine2016temporal,tarvainen2017mean,miyato2018virtual}, where strong augmentations and costly optimisation schemes are key for the outstanding performance. Both perspectives have shown  potentials but they still have limitations. Whilst energy models rely on hand-crafted features that are difficult to generalise, deep learning techniques lack of a well-defined~theory.

A  few recent works~\cite{iscen2019label,sellars2020two,sellars2021laplacenet} have  attempted to combine both perspectives,  so-called \textit{hybrid models}, where the principles of energy models and deep learning are combined. Hybrid models have demonstrated   performance which readily competes against deep learning techniques~\cite{sellars2021laplacenet}. However, similarly to deep learning techniques, hybrid models have been investigated mainly from the practical point of view. That is, in the context of hybrid models \textit{not much effort has been spent on developing better energy functionals and analysing their theoretical properties.} This is the motivation that drives the basis for this work.

More precisely, in this work we propose a robust graph energy for semi-supervised classification following a hybrid setting.
We focus on the  normalised Dirichlet energy~\eqref{model0} based on the graph Laplacian. Promising results have already been shown in this context. For example, the seminal algorithmic approach of~\cite{zhou2004learning} was introduced to perform a graph transduction, through the propagation of few labels, by minimising the Laplacian graph energy~\eqref{model0} for the specific case when $p=2$. Subsequent machine learning studies showed that using non-smooth energies with the $p=1$ norm, related to non-local total variation, can achieve better clustering performances~\cite{HH}, \textit{but original algorithms only approximated}~$p\to 1$.
More advanced optimisation tools were therefore proposed to consider the exact $p=1$ norm for binary~\cite{hein2013total} or multi-class~\cite{bresson2013multiclass} graph transduction. The normalisation of the operator is nevertheless crucial, as underlined in~\cite{vonLuxburg2007}, to ensure within-cluster similarity when the degrees $d_i$ of the nodes  are broadly distributed in the graph. These motivations drive our approach using the normalised Dirichlet energy~\eqref{model0} based on the graph~Laplacian.

\textbf{Contributions.} Our work is motivated by the problems of network calibration and the confirmation bias in pseudo-labelling~\cite{dawid1982well,niculescu2005predicting,guo2017calibration,arazo2020pseudo}, where one seeks that the probability of the predicted label reflects the ground truth correctness likelihood. In particular, in the context of deep learning
a current major family of techniques is pseudo-labelling. In this perspective,  a current challenge is how to improve poorly calibrated networks for better pseudo-labels~\cite{guo2017calibration,rizve2020defense}. In this work, we  argue that energy models can be a powerful alternative for inferring more meaningful pseudo-labels than the ones directly generated from a deep network.
With this motivation in mind, we propose a hybrid framework for semi-supervised classification called {CREPE Model (1-Lapla\textbf{C}ian g\textbf{R}aph \textbf{E}nergy for \textbf{P}seudo-lab\textbf{E}ls)}. The core of our proposal is a novel 1-Laplacian graph energy for inferring more certain pseudo-labels, these are then intertwined in an alternating optimisation scheme with a deep network for updating the graph. Our contributions are:

\medskip
\begin{compactitem}

   \item[\faPaperPlaneO]  We propose a hybrid framework for semi-supervised classification, in which we highlight:

    \smallskip
    $-$ A new  energy model based on the normalised and non-smooth Dirichlet energy \eqref{model0} based on the graph Laplacian, where we consider the exact    $p=1$ norm (energy functional~\eqref{func_2} following our scheme~\eqref{pde2}).
    Our functional is based on a carefully chosen class priors to enforce a sufficiently smooth solution, and to strengthen the intrinsic relation between the labelled and unlabelled~data.

    \smallskip
    $-$ We provide a convergence analysis of our model, and show that the solution trajectory  does indeed converge to a non-constant steady point (Proposition~\ref{prop_binary}, Proposition~\ref{prop_bin_conv}). Moreover, we provide a simple yet effective coupling constraint between labels for multi-class problems (Section~\ref{Sec::CouplingConst}).

    \smallskip
    $-$ We apply our results (Proposition~\ref{prop2}) and derive the connection of our energy model with the principles of pseudo-labelling. We then show how our graph based pseudo-labels can be iteratively updated, in an alternating optimisation scheme,  with a deep network.

    \item[\faPaperPlaneO] We extensively evaluate our model, with numerical and visual results, using six benchmarking  datasets from medical and natural images: CIFAR-10/100, Chest-Xray14, CBIS-DDSM,  Fashion-MNIST and Mini-ImageNet. We demonstrate that our  technique is able to generalise well to these diverse set-ups, and provides readily competing results against supervised techniques and state-of-the-art results for  semi-supervised classification.
\end{compactitem}

\begin{figure}[t!]
    \centering
    \includegraphics[width=0.48\textwidth]{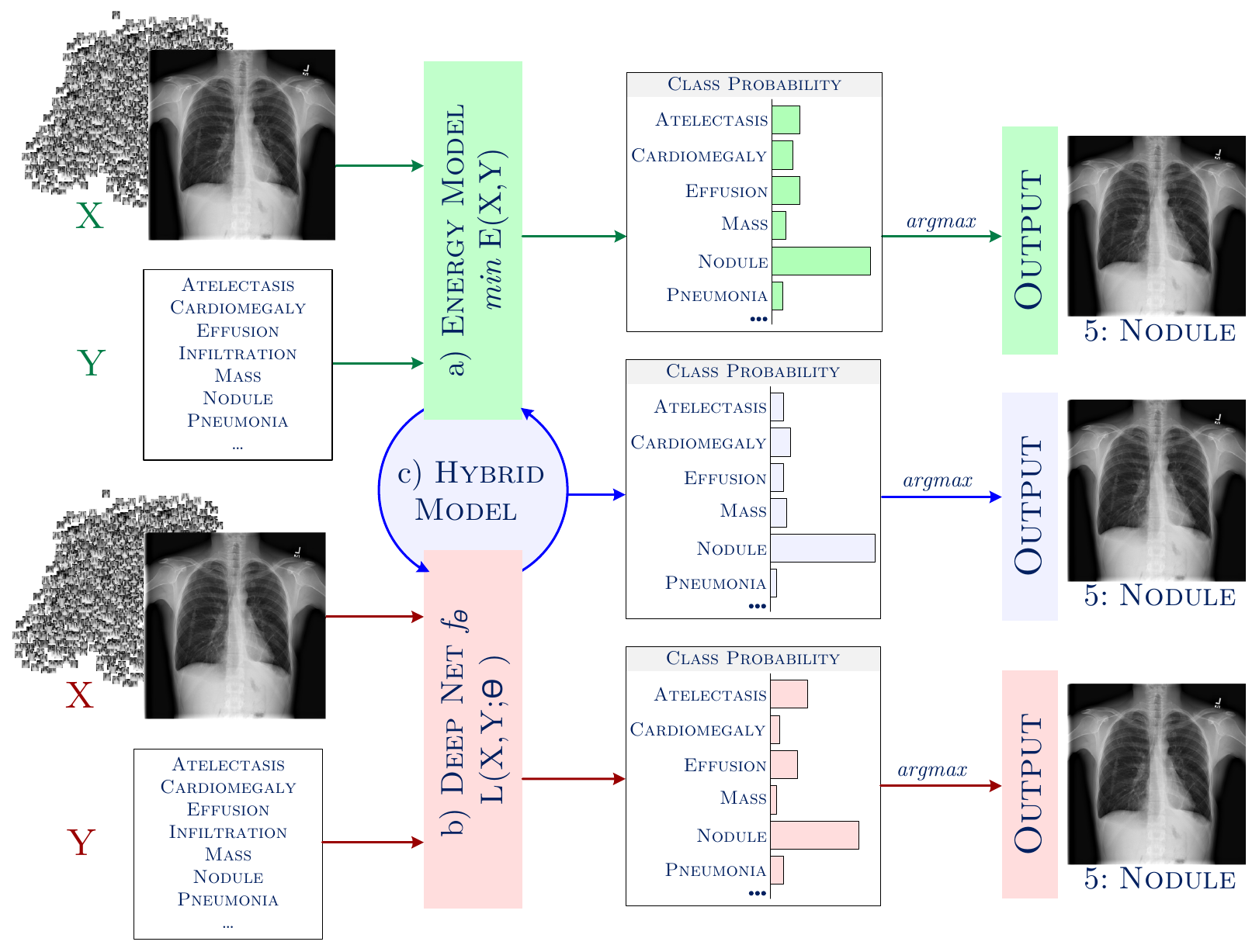}
    \caption{Three major semi-supervised categories in image classification. a) Energy models  seek to minimise a given energy (a maximum in probability) to infer the labels for the unlabelled set. b) Deep learning techniques aim to predict the unlabelled samples  solely relying on deep networks. c) Hybrid techniques use  principles from both energy models and deep networks. }
    \label{fig:sslcategories}
\end{figure}

\section{Related Work}
The problem of classifying images with scarce annotations has been extensively investigated in the machine learning community.
In the literature, semi-supervised learning (SSL)
can be broadly  divided into three categories: energy models (classic techniques) e.g.~\cite{zhu2002learning,zhou2004learning,hein2010inverse,zhang2011fast}, deep-learning based methods e.g.~\cite{laine2016temporal,tarvainen2017mean,miyato2018virtual, sohn2020fixmatch}, and hybrid techniques e.g.~\cite{iscen2019label,sellars2020two,sellars2021laplacenet}. An illustration of each category is displayed in Fig.~\ref{fig:sslcategories}.
These categories can be split by different perspectives including graph-based techniques, generative models, pseudo-labelling and consistency regularisation. In this section, we review the existing techniques.

\smallskip
\textbf{Energy Models for Semi-Supervised Classification.} Se\-mi-super\-vised classification has been extensively investigated in the literature, where the underpinning theory of this paradigm has been actively  developed since early works e.g.~\cite{merz1992semi,castelli1995exponential,vapnik1998statistical,joachims1999transductive,chapelle2006semi}. The solid foundations  have been strongly driven by the practical interest of relying less on labelled data in real-world applications such as text analysis~\cite{Nigam98,joachims1999transductive}.
A first family of techniques developed in the area is the energy models~\cite{chapelle2006semi,zhu2002learning,zhou2004learning}, where the main idea is to minimise a given energy (a maximum in probability) to infer the labels from the huge amount of unlabelled data using as prior a tiny labelled set. An illustration of this class of techniques is displayed in Fig.~\ref{fig:sslcategories}-a.
The term \textit{energy models} has been largely used in mathematics and physics for years, and since the early developments in semi-supervised learning e.g.~\cite{chapelle2006semi,zhu2002learning,zhou2004learning}.
There are several perspectives under this family of techniques including
generative models e.g.~\cite{kemp2004semi,grandvalet2005semi,adams2009archipelago} and low-density separation approaches e.g.~\cite{joachims1998making,joachims1999transductive,xu2008efficient,vapnik2013nature}. Besides these techniques another large subfamily of techniques is graph based approaches which is the focus of our interest.

Several techniques have been reported following the graph perspective including  random walks e.g.~\cite{szummer2002partially,zhou2004learningb}, harmonic based energy e.g.~\cite{zhu2002learning}, graph mincut e.g.~\cite{blum2001learning,joachims2003transductive,blum2004semi}, and spectral techniques e.g~\cite{belkin2002semi,chapelle2003cluster}. In most recent works,  the authors of~\cite{jung2016semi} used a sparse variant of label propagation under the condition that initial labels are in the proximity of the cluster boundaries. A weighted nonlocal Laplacian energy was introduced in~\cite{shi2017weighted}, where the authors enforce preservation of the symmetry of the Laplace operator. A kernel clustering approach was used in~\cite{mai2018random,mai2021consistent}  as an approach for Laplacian regularisation. The Poisson equation on a graph was used in~\cite{calder2020poisson}  for low label rates  classification.

\smallskip
\textbf{Deep Semi-Supervised Techniques.}
The power of deep learning has been recently applied for semi-supervised classification, which leads the current state-of-the-art performance. A visualisation of this family of techniques is displayed in Fig.~\ref{fig:sslcategories}-b).
There exist two major families of techniques in modern semi-supervised techniques: \textit{consistency regularisation} (aka perturbation-based methods) e.g.~\cite{laine2016temporal,sohn2020fixmatch,xie2021muscle} and \textit{pseudo-labelling} e.g.~\cite{lee2013pseudo,arazo2020pseudo,hu2021simple}. Consistency regularisation techniques work under the assumption that the model's performance (output $g(X_u)$, where $X_u$ is the unlabelled data) should not change under any induced $\tau$-perturbation -- that is: $g(X_u) =  g(X_u+\tau)$. Following this principle, several techniques have been proposed including the works of~\cite{laine2016temporal,tarvainen2017mean,miyato2018virtual,athiwaratkun2018there,verma2019interpolation,ke2019dual,hu2021simple}. A major challenge on these approaches is  how to set $\tau$. Different strategies for $\tau$ have been considered in the literature,  including mixup augmentations e.g.~\cite{zhang2018mixup}, generative augmentations e.g.~\cite{miyato2018virtual} and SOTA augmenters e.g.~\cite{cubuk2018autoaugment, cubuk2020randaugment}. The core of the performance, of this family of technique, is the use of costly optimisation schemes (e.g. more than 1M training iterations) along  with strong augmentations.

The second large family of deep semi-supervised techniques is pseudo-labelling introduced by Lee in~\cite{lee2013pseudo}. The idea of pseudo-labelling is to generate proxy labels to guide the learning process. Different techniques have been proposed to improve the performance of pseudo-labelling. The use of mix-max feature regularisation was presented in~\cite{shi2018transductive}. The authors of~\cite{li2020density} proposed a density aware mechanism for improving feature learning and pseudo-label generation. Label propagation using the graph Laplacian with the $p=2$ case have been proposed in~\cite{iscen2019label} and in combination with clustering regularisation in~\cite{sellars2020two}. Mixup has been shown to offer good performance along with small labels per mini-batch~\cite{arazo2020pseudo}, and together with graph based pseudo-labels~\cite{sellars2021laplacenet}. Certainty mechanisms have also been proposed  to improve pseudo-labelling~\cite{rizve2020defense,hu2021simple}.

\smallskip
\textbf{Hybrid Techniques and Comparison to our Work.} Whilst existing techniques either are energy models or deep learning techniques, works simultaneously using these principles, called ~\textit{hybrid techniques}, are very recent and scarce (see Fig.~\ref{fig:sslcategories}-c).
The existing works are under the family of pseudo-labelling techniques, where energy models have been used for improving performance.
The work of~\cite{iscen2019label} adapted the energy model of~\cite{zhou2004learning} to an inductive framework using modern deep networks. The same energy model was used in~\cite{sellars2020two} along with clustering regularisation. Most recently, the work of \cite{sellars2021laplacenet} showed high performance by using the energy model of~\cite{zhou2004learning} along with a multi-sampling augmentation strategy.

The existing hybrid models have a commonality that is the use of the energy model of~\cite{zhou2004learning}, where one seeks to minimise~\eqref{model0} for the specific case $p=2$. The focus of existing works is not the energy model part but rather the development of mechanisms for improving the network performance. In contrast to those works, our approach centers in developing better energy functionals and their theoretical properties to help the learning process. Moreover, existing works have not investigated how more robust energy models impact the performance.
\section{Preliminaries}
This  work addresses the problem of semi-supervised classification. In particular, we follow the graph based perspective in semi-supervised learning. Formally, we aim at solving the following problem.

\textbf{Problem Statement.} Given a set of samples $X=(x_1,...,x_l,x_{l+1},...,x_n) \text{ where } x_i\in\mathcal{X}$, we assume that a tiny subset is labelled $D_L =\{ (x_i ,y_i)  \}_{i=1}^{l}$ with provided labels  $\{y_i\}_{i=1}^{l} \in \mathcal{L}= \{1,..,L\}$ for $L$ classes, and a large subset is unlabelled  $X_u =\{ x_i \}_{i=l+1}^{n}$ .
We then seek to infer a function $f: \mathcal{X} \mapsto \mathcal{L}$ such that $f$ gets a good estimate for $\{ x_i \}_{i=l+1}^{n}$ with minimum generalisation error.

\begin{figure}[t!]
    \centering
    \includegraphics[width=0.48\textwidth]{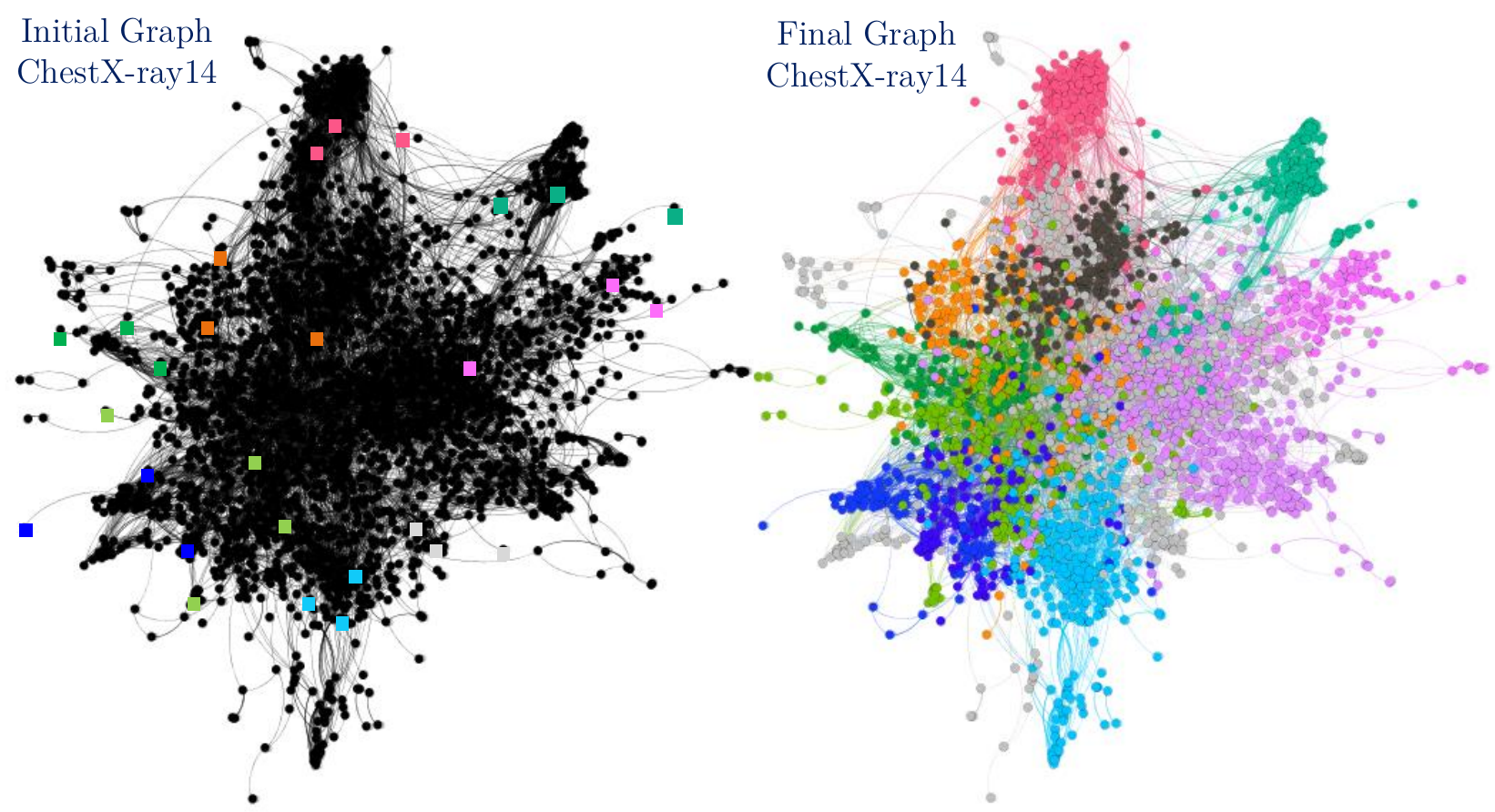}
    \caption{Visual illustration of our graph based energy that seeks to infer better pseudo-labels. Left figure displays the initial graph where only a tiny label set is given as prior. Each colour represents a different class (few given labels per class, e.g. \crule[blue]{0.2cm}{.2cm}='Mass', \crule[orange]{0.2cm}{.2cm}='Nodule', using the ChestXray-14 dataset) and black nodes denotes the unlabelled data.  Right side shows the inferred pseudo-labels using our energy model.}
    \label{fig:graphVis}
\end{figure}

We address this problem from the hybrid perspective (see Fig.~\ref{fig:sslcategories}-c), where one seeks to  combine principles from energy models and deep networks. In particular, our work focus on a hybrid technique from the pseudo-labelling perspective. In deep semi-supervised classification for pseudo-labelling, the main goal is to solve a loss that relates the labelled  and unlabelled  sets,  whose general form reads:
\begin{equation} \label{generalSSL}
  \min_{\theta}{{\sum_{(x,y)\in D_L }\mathcal{L_S}(x,y;\theta)}} + \gamma {{\sum_{x,\in X_u } \mathcal{L_U}(x,\hat{y};\theta)}},
\end{equation}

\noindent
where the two terms  $\mathcal{L_S}$ and $\mathcal{L_U}$ handle the labelled and unlabelled set respectively, $\theta$ is the network parameters, $\hat{y}$ are pseudo-labels,  and $\gamma$ a positive parameter weighting the importance of each term. In the body of literature, the main difference between existing works is the way to define $\mathcal{L_U}$, for example using a pseudo-labelling strategy or consistency regularisation.

A current major challenge is how to better handle issues relating to network calibration and confirmation bias e.g.~\cite{dawid1982well,niculescu2005predicting,guo2017calibration,arazo2020pseudo}. In the context of pseudo-labelling, hybrid techniques e.g.~\cite{iscen2019label,sellars2020two,sellars2021laplacenet} have shown that one can mitigate such issues by inferring pseudo-labels from an energy model and then combine them with deep networks rather than predict them directly from a deep network. However, existing hybrid techniques have only focused on designing  mechanism for the networks and the investigation of better energy models are to be investigated.

\textit{With previous motivation in mind, we seek to design better funtionals for improving the inference of pseudo-labels.} To do this, we follow a  graph based perspective for inferring more certain pseudo-labels. In this work, we consider functions $u\in\R^n$, defined over a set $\mathcal{N}$ of $n$ nodes. Our main points of interest  are convex and absolutely $p$-homogeneous (i.e. $J(\alpha u)=|\alpha|^pJ(u)$) non-local functionals, defined on $u$, of the particular form:

\begin{equation}\label{model0}
D_p(u)=\sum_{ij}w_{ij}\left|\left| \frac{u_i}{d_i^{1/p}}-\frac{u_j}{d_j^{1/p}}\right|\right|^p,
\end{equation}

\noindent
with weights $w_{ij}=w_{ji}\geq 0$ taken such that the vector $d\in\R^n$ has non-null entries satisfying: $d_i=\sum_jw_{ij}>0$.
This energy acts on the graph defined by nodes $\mathcal{N}$ and weights $w_{ij}$.
With respect to classical Dirichlet energies associated to the graph $p$-Laplacian \cite{andreu2008nonlocal,elmoataz2008nonlocal,hein2013total,bresson2013multiclass}, it includes a normalisation through rescaling with the degree of the node.
In this work, we focus our attention to the non smooth case $p=1$ with  the absolutely one homogeneous energy defined by the function $J(u)=D_1(u)$ that can be rewritten as:

\begin{equation}\label{model}J(u)=\normu{WD^{-1}u},
\end{equation}
with an $n\times n$ diagonal matrix $D=\diag(d)$, containing the nodes degree so that $d=D\Un_n$, and an $m\times n$ matrix  $W$ that encodes the $m$ edges in the graph. Each of these edges is represented on a different line of the sparse matrix $W$ with the value $w_{ij}$ (resp. $-w_{ij}$) on the column $i$ (resp. $j$).

\smallskip
\textbf{Subdifferential.}
Let us first define as
$\partial J$ the set of possible subdifferentials of $J$: $\partial J=\{\subgradp,\, \textrm{ s.t. } \exists u,\, \textrm{ with } \subgradp\in\partial J(u)\}$.
Any  absolutely one homogeneous function $J$  checks:
\begin{equation}\label{proper_1hom}
\begin{split}
J(u)&=\sup_{\subgradp\in\partial J} \langle \subgradp,u\rangle\\
\end{split}
\end{equation}

\noindent
so that $J(u)=\langle \subgradp,u\rangle,\, \forall \subgradp\in\partial J(u)$.

For the particular function $J$ defined in \eqref{model}, we can observe that

\begin{equation}\label{sousdiff}
\subgradp\in\partial J\Leftrightarrow \subgradp=D^{-1}W^\top z,\,\textrm{ with }||z||_\infty\leq 1.
\end{equation}

\noindent
Considering the finite dimension setting, there exists $L_J<\infty$ such that $\normd{\subgradp}<L_J$, $\forall \subgradp\in\partial J$.
We also have the following property.

\begin{prop}\label{prop_sub}
For all $\subgradp\in \partial J$, with $J$ defined in \eqref{model}, one has
$$\langle \subgradp,d\rangle=0.$$
\end{prop}
\begin{proof}
Observing that  $d=D\Un$ and using \eqref{sousdiff} we have that $\exists z\in\R^m$ such that
$$\langle \subgradp,d\rangle= \langle D^{-1}W^\top z, D\Un_n\rangle.$$
Since the weights satisfy $w_{ij}=w_{ji}$, then for all $z\in\R^m$:
\begin{equation*}
\begin{split}
&\langle W^\top z,\Un_n\rangle=\sum_i \sum_j w_{ij}(z_i-z_j)= \\
&\sum_{i}\sum_{j>i}w_{ij}(z_i-z_j-z_i+z_j)=0.
\end{split}
\end{equation*}
\end{proof}

\smallskip
\textbf{Eigenfunctions.}
Eigenfunctions of any convex functional $J$ satisfy $\lambda u\in\partial J(u)$.
For $J$ being the nonlocal total variation, (i.e. when $d_i$ is constant), eigenfunctions are known to be essential tools to provide a relevant clustering of the graph~\cite{vonLuxburg2007}.
Methods \cite{HH,NIPS2012_4726,bresson2013adaptive,benning2017learning,AGP18} have thus been designed to estimate such eigenfunctions through the local minimisation of the Rayleigh quotient, which reads:
\begin{equation}\label{RQ}\min_{\normd{u}=1}\frac{J(u)}{H(u)},\end{equation}
with another absolutely one homogeneous function $H$, that is typically a norm.
Taking $H(u)=\normd{u}$ as the $\ell_2$ norm, one can recover eigenfunctions of $J$.
For $H(u)=\normu{u}$ being the $\ell_1$ norm, one can also compute bi-valued functions $u$ that are local minima  of \eqref{RQ} and eigenfunctions of $J$  \cite{Feld}.
Being bivalued, these estimations can easily be used to realise a partition of the domain. These schemes also relate to the Cheeger cut of the graph induced by nodes $u_i$ and edges $w_{ij}$. Balanced cuts can also be obtained by considering $H(u)=\normu{u-\median(u)}$  \cite{bresson2013multiclass}.

A last point to underline comes from Proposition \ref{prop_sub} that states that eigenfunctions $\lambda u\in\partial J(u)$ should be orthogonal to $d$. It is thus important to design schemes that ensure this property.

\begin{figure}[t!]
    \centering
    \includegraphics[width=0.48\textwidth]{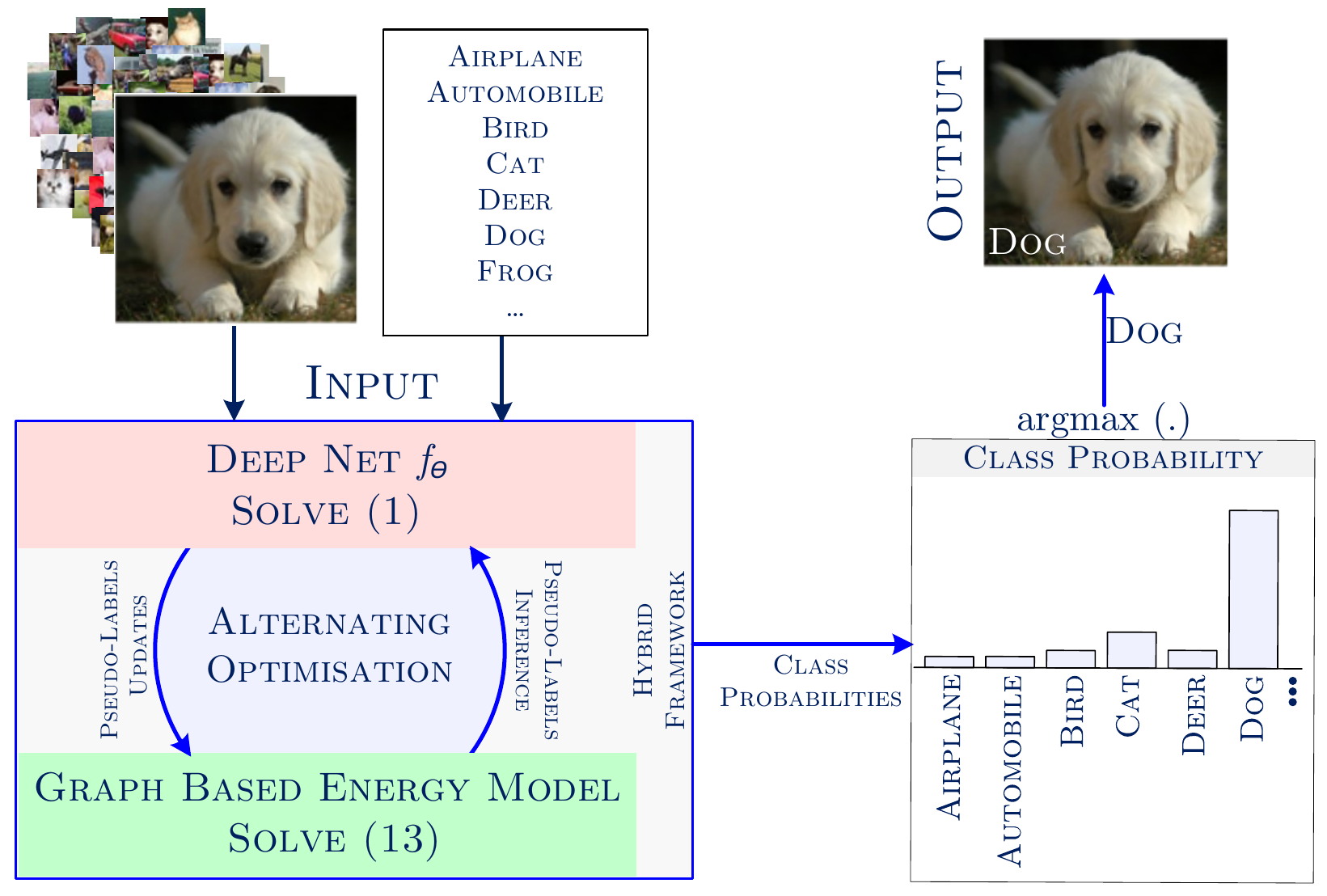}
    \caption{Visual illustration of our hybrid framework. The core of our technique is a graph based energy model for inferring better pseudo-labels (see green box). These are then updated in an alternating optimisation fashion through a deep net (red box) for boosting the classification performance. The output is assigned as the one with the highest probability. }
    \label{fig:hybridFramework}
\end{figure}

\section{Improving Pseudo-Labelling with the  1-Laplacian Graph Energy}
This section describes our proposed energy model that fits into a hybrid framework that we called {CREPE Model (1-Lapla\textbf{C}ian g\textbf{R}aph \textbf{E}nergy for \textbf{P}seudo-lab\textbf{E}ls)}. As illustrated in Fig.~\ref{fig:hybridFramework}, there are two main components in hybrid techniques: a deep network and an energy model. In this work and unlike existing hybrid techniques that focused on designing better mechanisms for improving the network performance, we focus on developing better energy functionals and their theoretical properties (see green box  from Fig.~\ref{fig:hybridFramework}).  This section describes three key parts: i) the convergence of our energy model, ii) the definition of our coupling constrain for the multi-class problem and iii) our multi-class energy flow for pseudo-labelling.

\textit{Core Idea.} We seek to infer better pseudo-labels using an energy model (outside a deep network) instead of generating them directly from a deep network.  To do this, we introduce a 1-Laplacian graph energy (see green box  from Fig.~\ref{fig:hybridFramework}), which is detailed next.

\subsection{Convergence Analysis}
In the following,  we will denote the value of function $u$ at node $x$ by $u(x)$ and the value of $u$ at iteration $k$ as $u_k$.
In order to realise a binary partition of the domain of the graph $\mathcal{N}$ through the minimisation of the quotient $R(u)=J(u)/H(u)$, we adapt the method of~\cite{Feld} to incorporate the scaling $d(x)$ of \eqref{model} and consider the semi-explicit scheme:

\begin{equation}\label{pde}
\left\{\begin{array}{ll} \frac{u_{k+1/2}-u_k}{\delta t}&=\frac{J(u_k)}{H(u_k)} (q_k-\tilde q_k)-\subgradp_{k+1/2},\\
u_{k+1}&=\frac{u_{k+1/2}}{\normd{u_{k+1/2}}}
\end{array}\right.
\end{equation}
with $\subgradp_{k+1/2}\in\partial J(u_{k+1/2})$,  $q\in\partial H(u_k)$,  $\tilde q_k=\frac{\langle d,q_k\rangle}{\langle d,d\rangle}d$. We recall that both $J$ and $H$ are absolutely one homogeneous and satisfy \eqref{proper_1hom}.
Since $\langle \subgradp,d\rangle=0$, $\forall \subgradp\in\partial J$, the shift with $\tilde q_k$ is necessary to show the convergence of the scheme \eqref{pde}  as we have  $u_k\to u^*\Rightarrow \frac{J(u^*)}{H(u^*)} (q^*-\tilde q^*)=\subgradp^*$, for $\subgradp^*\in\partial J(u^*)$ and $q^*\in\partial H(u^*)$. \\
Such sequence $u_k$ satisfies the following properties.

\begin{prop}\label{prop_binary}
For $\langle u_0,d\rangle=0$, the trajectory $u_k$ given by \eqref{pde} satisfies:

\renewcommand{\labelenumi}{\AngieItem{\arabic{enumi}}}
\begin{enumerate}\setlength{\itemsep}{5pt}
\item  $\langle u_{k+1},d\rangle=0$, \AngieHighlight
\item $||u_{k+1/2}||_2\geq||u_k||_2$, \AngieHighlight
\item $R(u_k)$ is non increasing,  \AngieHighlight
\item $H(u_{k+1/2})\leq \kappa<+\infty$.  \AngieHighlight
\end{enumerate}
\end{prop}

\smallskip
\begin{proof}
In this proof, we  use the fact that $u_{k+1/2}$ defined in \eqref{pde} is the unique minimiser of:

\begin{equation}\label{func_1}F_k(u)=\frac1{2\delta t}\normd{u- u_k}^2+R(u_k)\langle q_{k}-\tilde q_k,u\rangle+J(u).
\end{equation}

\medskip
For $\langle u_k,d\rangle=0$, we have
{\small
\begin{equation*}
\begin{split}
\langle u_{k+1/2},d\rangle&=\langle u_k,d\rangle+\delta t\left(R(u_k) \langle (q_k-\tilde q_k),d\rangle-\langle \subgradp_{k+1/2},d\rangle\right)\\
&=\delta t R(u_k)\left( \langle q_k,d\rangle-\frac{\langle d,q_k\rangle}{\langle d,d\rangle}\langle d,d\rangle\right)\\
&=0,
\end{split}
\end{equation*}}
where we used Proposition 1 in the right part of the previous relation to get $\langle \subgradp_{k+1/2},d\rangle=0$. \AngieHighlight
We conclude with the fact that $u_{k+1}$ is a rescaling of $u_{k+1/2}$.

\medskip
Since $H$ is a norm, it is absolutely one homogeneous and $q_k\in \partial H(u_k)\Rightarrow H(u_k)=\langle q_k,u_k\rangle$. Next, we observe that  $J(u_k)=\sup_{\subgradp\in\partial J}\langle \subgradp,u_{k}\rangle\geq  \langle \subgradp_{k+1/2},u_{k}\rangle$ and we get
{\small
\begin{equation*}
\begin{split}
\langle u_{k+1/2},u_k\rangle&=||u_k||^2_2+\delta t\left( R(u_k) \langle q_k-\tilde q_k,u_k\rangle-\langle \subgradp_{k+1/2},u_k\rangle\right)\\
&\geq||u_k||^2_2+\delta t\left(J(u_k) - R(u_k)\langle \tilde q_k,u_k\rangle-J(u_k)\right)\\
&\geq||u_k||^2_2-\delta tR(u_k)\frac{\langle d,q_k\rangle}{\langle d,d\rangle} \langle d,u_k\rangle\\
&\geq||u_k||^2_2.
\end{split}
\end{equation*}
}

We then conclude with the fact that $\langle u_{k+1/2},u_k\rangle\leq ||u_{k+1/2}||_2.||u_k||_2$.

\medskip
Since $\langle u_k,d\rangle=0$ for all $k$ and $\tilde q_k=\frac{\langle d,q_k\rangle}{\langle d,d\rangle} d$, then $\langle\tilde q,u_{k+1/2}\rangle=\langle\tilde q,u_{k}\rangle=0$. Next, we recall that $H(u_{k+1/2})=\sup_{q\in\partial H{.}}\langle q,u_{k+1/2}\rangle\geq  \langle q_{k},u_{k+1/2}\rangle$. Hence we have
{\small
\begin{equation}
\begin{split}
F_k(u_{k+1/2})\leq F(u_k)\\
\frac1{2\delta t}\normd{u_{k+1/2}- u_k}^2- R(u_k)\langle q_{k},u_{k+1/2}\rangle+J(u_{k+1/2})\leq 0\\
\frac1{2\delta t}\normd{u_{k+1/2}- u_k}^2+J(u_{k+1/2})\leq R(u_k)H(u_{k+1/2})\label{ineq}\\
R(u_{k+1/2})\leq R(u_k)\\
R(u_{k+1})\leq R(u_k)\\
\end{split}
\end{equation}
}
where the final rescaling with $\normd{u_{k+1/2}}$ is possible since $J$ and $H$ are absolutely one homogeneous functions.

\medskip
In the finite dimension setting,  there exists $\const_J,\const_H<\infty$ such that $||p||\leq \const_J$ and $||q||\leq \const_H$ for an  absolutely one homogeneous functionals $J$ defined in (1) and a norm $H$. Then one has

{\small
\begin{equation*}
\begin{split}
     u_{k+1/2}&=u_k+\delta t\left(\frac{J(u_k)}{H(u_k)} (q_k-\tilde q_k)-\subgradp_{k+1/2}\right)\\
     \normd{u_{k+1/2}}^2&=\langle u_k,u_{k+1}\rangle +\delta t \left(\frac{J(u_k)}{H(u_k)} \langle q_k,u_{k+1/2}\rangle-\langle \subgradp_{k+1/2}\rangle\right)\\
     \normd{u_{k+1/2}}^2&\leq \normd{u_{k+1/2}}\left(\normd{u_{k}}+\delta t\left(\frac{J(u_k)}{H(u_k)}\const_H+\const_J\right)\right)\\
     \normd{u_{k+1/2}}&\leq 1+\delta t\left(\frac{J(u_0)}{H(u_0)}\const_H+\const_J\right).
\end{split}
\end{equation*}
}
From the equivalence of norms in finite dimensions, there exists $0<\kappa<\infty$ such that $H(u_{k+1/2})\leq \kappa$.  
\end{proof}

\noindent
Hence, we can show the convergence of the trajectory.

\begin{prop}\label{prop_bin_conv}
The sequence $u_k$ defined in \eqref{pde} converges to a non-constant steady point $u^*$.
\end{prop}
\begin{proof}
As $u_{k+1/2}$ is the unique minimiser of $F_k$ in \eqref{func_1}, as $F_k(u_k)=0$, and as $\langle q_k-\tilde q _k,u_{k+1/2}\rangle\leq H(u_{k+1/2})$, we get:

\begin{equation}\label{tosum}
\frac1{2\delta t H(u_{k+1/2})}\normd{u_{k+1/2}- u_k}^2+R(u_{k+1})\leq R(u_{k}).
\end{equation}
Since $u_{k+1}$ is the orthogonal projection of $u_{k+1/2}$ on the $\ell_2$ ball then $\normd{u_{k+1}- u_k}\leq\normd{u_{k+1/2}- u_k}$.
Finally, from statement 4 of Proposition \ref{prop_binary}, we have that $1/H(u_{k+1/2})\geq 1/\kappa$.
We then sum relation \eqref{tosum} from $0$ to $K$ and deduce that:

$$\sum_{k=0}^K\frac1{2\delta t\kappa}\normd{u_{k+1}- u_k}^2\leq H(u_0),$$
so that $\normd{u_{k+1}- u_k}$ converges to $0$. Since all the quantities are bounded, we can show that up to a subsequence $u_k\to u^*$ (see \cite{Feld}, Theorem 2.1).

From Proposition \ref{prop_binary}, the points $u_k$ being of constant norm and $\langle d,u_k\rangle$ being zero (with positive weights $d_i$), the limit point $u^*$ of the trajectory \eqref{pde} necessarily has negative and positive entries.
\end{proof}

In practice, to realise a partition of the graph with the scheme \eqref{pde}, we  miniminise the functional \eqref{func_1} at each iteration $k$ with the primal-dual algorithm of~\cite{CP11} to obtain $u_{k+1/2}$, and then normalise this estimation.
As it is non-constant and satisfies $\langle u^*,d\rangle=0$, the limit of the scheme $u^*$ can be used for partitioning with the simple criteria $u^*>0$.

\subsection{Our Coupling Constraint}\label{Sec::CouplingConst}
As we consider a multi-class setting,  we  aim at finding $\valL$ coupled functions $u^\indexL$ that are all local minima of the ratio $J(u)/H(u)$. The issue is to define a  coupling constraint between the $u^\indexL$'s such that it is easy to project on. Let  ${\bf u}=[u^1,\cdots u^\valL]$, in this work we consider the following simple linear coupling, which reads:

\begin{equation}\label{constr_coupling}C:\{{\bf u},\, \textrm{s.t. } \sum_{\indexL=1}^\valL u^\indexL(x)=0,\, \forall x\in\mathcal{N}\}.\end{equation}

There are three main reasons for considering such coupling instead of classical simplex \cite{bresson2013multiclass,rangapuram2014tight,gao2015medical} or orthogonality \cite{dodero2014group} constraints:
\renewcommand{\labelenumi}{\AngieItemRed{\arabic{enumi}}}
\begin{enumerate} \setlength{\itemsep}{5pt}
    \item Projection on this linear constraint is explicit with a simple shift of the vector ${\bf u} (x)$ for each node $x$. On the other hand, simplex constraint ($u^\indexL(x)\geq 0$, $\sum_{\indexL} u^\indexL(x)=1$, $\forall x$) requires more expensive projections of the vectors ${\bf u} (x)$ on the $\valL$ simplex. Lastly, projection on the orthogonal constraint of the  $u^\indexL$'s is a non convex problem.  \smallskip \AngieHighlightB
    \item Contrary to the simplex constraint, it is compatible with the weighted zero mean condition $\langle u^\indexL,d\rangle$ that any eigenfunction of $J$ should satisfy, as shown in Proposition \ref{prop_sub}. \smallskip \AngieHighlightB
    \item The characteristic function of a linear constraint is absolutely one homogeneous. This leads to a natural extension of the binary case. \AngieHighlightB
\end{enumerate}

\subsection{Multi-Class Flow for Better Pseudo-Labelling}
In previous section, we provide the convergence analysis and coupling constraint of our energy model. In this section, we detail how these elements fit into our  new energy functional for pseudo-labelling.

We recall to the reader that we  consider the problem:

\begin{equation}\label{RQ2}\min_{\normd{{\bf u}}=1}\sum_{\indexL=1}^\valL\frac{J(u^\indexL)}{H(u^\indexL)}.\end{equation}

To find a local minima of \eqref{RQ2}, we  define our iterative multi-class energy functional, which reads:
\begin{multline}\label{func_2}
F^\valL_k({\bf u})=\frac1{2\delta t}\normd{{\bf u}-{\bf u}_k}^2-\sum_{\indexL=1}^\valL R(u_k^\indexL)\langle q_{k}^\indexL-\tilde q_k^\indexL,u^\indexL\rangle \\ +\sum_{\indexL=1}^\valL J(u^\indexL)+\chi_C({\bf u})
\end{multline}
where $q_k^\indexL\in \partial H(u_k^\indexL)$ and $\chi_C$ is the characteristic function of the constraints  \eqref{constr_coupling}.
Starting from an initial point ${\bf u}_0$ that satisfies the constraint ($\chi_C({\bf u}_0)=0$) and has been normalised ($\normd{{\bf u}_0}^2=\sum_{\indexL=1}^\valL\normd{u_0^\indexL}^2=1$), the scheme we consider reads:

{
\begin{equation}\label{pde2}
\left\{\begin{array}{ll}
{\bf u}_{k+1/2}&=\uargmin{{\bf u}}F^\valL_k({\bf u})\vspace{0.2cm}\\
{\bf u}_{k+1}&=\frac{{\bf u}_{k+1/2}}{\normd{{\bf u}_{k+1/2}}}.
\end{array}\right.
\end{equation}}

In practice, if for some $\indexL$, $u^\indexL_{k+1/2}$ vanishes, then we define $R(u^\indexL_{k+1})=0$ for the next iteration. With such assumptions, the sequence ${\bf u}_k$ have the following properties.

\begin{prop}\label{prop2}
For $\langle u_0^\indexL,d\rangle=0$, $\indexL=1\cdots \valL$, the trajectory ${\bf u}_k$ given by \eqref{pde2} satisfies
\renewcommand{\labelenumi}{\AngieItem{\arabic{enumi}}}
\begin{enumerate} \setlength{\itemsep}{5pt}
\item  $\langle u_k^\indexL,d\rangle=0$, \AngieHighlight
\item $||{\bf u}_{k}||_2\leq||{\bf u}_{k+1/2}||_2\leq \kappa<\infty$, \AngieHighlight
\item $\sum_{\indexL=1}^\valL H(u_{k+1}^\indexL) \left( R(u_{k+1}^\indexL) - R(u_k^\indexL)  \right)$ \\
\hspace{2cm} $\leq -\frac1{2\delta t\kappa} \normd{{\bf u}_{k+1}-{\bf u}_k}^2$. \AngieHighlight
\end{enumerate}
\end{prop}

\begin{proof}
The scheme reads
\begin{multline*}
\left\{\begin{array}{ll}
{\bf u}_{k+1/2}=\uargmin{{\bf u}}F^\valL_k({\bf u}):=\frac1{2\delta t}\normd{{\bf u}-{\bf u}_k}^2 +\sum_{\indexL=1}^\valL J(u^\indexL) \\ \hspace{0.5cm}-\sum_{\indexL=1}^\valL R(u_k^\indexL)\langle q_{k}^\indexL-\tilde q_k^\indexL,u^\indexL\rangle+\chi_C({\bf u})\vspace{0.3cm}\\
{\bf u}_{k+1}=\frac{{\bf u}_{k+1/2}}{\normd{{\bf u}_{k+1/2}}}.
\end{array}\right.
\end{multline*}
The Karush–Kuhn–Tucker conditions of the above problem states that there exist  $\subgradp^\indexL_{k+1/2}$ and $r_{k+1/2}$ such that
$$u_{k+1/2}^\indexL=u_k^\indexL+\delta t\left(R( u^\indexL_k)(q_{k}^\indexL -\tilde q^\indexL_k)-\subgradp^\indexL_{k+1/2}-r_{k+1/2}\right)$$
where $\subgradp^\indexL_{k+1/2}\in\partial J(u^\indexL_{k+1/2})$ and $r_{k+1/2}$ is a Lagrange multiplier independent of $\indexL$ for the linear constraint $\chi_C$. The point ${\bf u}_{k+1/2}$ in the above scheme corresponds to  the global minimiser of  $F^\valL_k({\bf u})$.

\renewcommand{\labelenumi}{\AngieItem{\arabic{enumi}}}
For $\langle u^\indexL_k,d\rangle=0$, and following point 1 of Proposition 2, we have  \AngieHighlight
\begin{equation*}
\begin{split}
\langle u_{k+1/2}^\indexL,d\rangle&=\langle u_k^\indexL,d\rangle+\delta t\big(R(u_k^\indexL) \langle (q_k^\indexL-\tilde q_k^\indexL),d\rangle\\
&-\langle \subgradp^\indexL_{k+1/2},d\rangle-\langle r_{k+1/2},d\rangle \big)\\
&=-\langle r_{k+1/2},d\rangle.
\end{split}
\end{equation*}
Next, as ${\bf u}_{k+1/2}\in C$, we have $\sum_\indexL u^\indexL_{k+1/2}(x)=0$, $\forall x\in\mathcal{N}$ and obtain:
\begin{equation*}
\begin{split}
\sum_{\indexL=1}^\valL\langle u_{k+1/2}^\indexL,d\rangle&=-\sum_{\indexL=1}^\valL\langle r_{k+1/2}, d\rangle\\
\sum_{\indexL=1}^\valL\sum_{x\in\mathcal{N}} u^\indexL _{k+1/2}(x)d(x)&=-\valL\langle r_{k+1/2}, d\rangle\\
\sum_{x\in\mathcal{N}} d(x)\left(\sum_{\indexL=1}^\valL u^\indexL _{k+1/2}(x)\right) &=-\valL\langle r_{k+1/2}, d\rangle\\
0&=\langle r_{k+1/2}, d\rangle.
\end{split}
\end{equation*}
\medskip
We have
\begin{equation*}
\begin{split}
\langle u^\indexL_{k+1/2},u^\indexL_k\rangle=||u^\indexL_k||^2_2+\delta t\left( R(u^\indexL_k) \langle q^\indexL_k-\tilde q^\indexL_k,u^\indexL_k\rangle-\right.\\
\left. \langle \subgradp^\indexL_{k+1/2},u^\indexL_k\rangle-\langle r_{k+1/2},u^\indexL_k\rangle\right).
\end{split}
\end{equation*}

\medskip
We follow the point 2 of Proposition 2 to first get:
$\langle u_{k+1/2}^\indexL,u_{k}^\indexL\rangle\geq \normd{u_k^\indexL}-\langle r^\indexL_{k+1/2},u_k^\indexL\rangle$, for $i=1\cdots n$. Then, as $\sum_\indexL  \langle r_{k+1/2},u^\indexL_k\rangle=  \sum_x  r_{k+1/2}(x)\sum_\indexL  u^\indexL_k(x)=0$, we deduce that
$\normd{{\bf u}_{k+1/2}}.\normd{{\bf u}_{k}}\geq  \sum_\indexL  \langle u^\indexL_{k+1/2},u^\indexL_k\rangle\geq\sum_\indexL  \langle u^\indexL_{k},u^\indexL_k\rangle=||{\bf u}_k||_2^2$.
Next we observe that

\begin{equation*}
\begin{split}
\normd{u^\indexL_{k+1/2}}^2 &=\langle u^\indexL_{k+1/2},u^\indexL_k\rangle+\delta t\left(-\langle r_{k+1/2},u^\indexL_{k+1/2}\rangle\right.\\
&\left.+R(u^\indexL_k) \langle q^\indexL_k-\tilde q^\indexL_k,u^\indexL_{k+1/2}\rangle -J(u^\indexL_{k+1/2})  \right).
\end{split}
\end{equation*}
Summing on $\indexL$, we get
\begin{equation*}
\begin{split}
\normd{{\bf u}_{k+1/2}}^2&\leq \normd{{\bf u}_{k+1/2}}\left( \normd{{\bf u}_{k}} +\delta t\left( \sum_{\indexL=1}^\valL R(u^\indexL_k) \normd{q_{k}^\indexL} \right.\right.\\
&\left.\left.   +\normd{\subgradp^\indexL_{k+1/2}} \right)\right)\\
\normd{{\bf u}_{k+1/2}}&\leq \normd{{\bf u}_{k}} +\delta t\left( \sum_{\indexL=1}^\valL \frac{J(u^\indexL_k)}{H(u^\indexL_k)} \const_H+\const_J\right)\leq 1  \\
&+\delta t\const_J\left( \sum_{\indexL=1}^\valL \frac{\normd{u^\indexL_k}}{H(u^\indexL_k)} \const_H+1\right).
\end{split}
\end{equation*}

Notice that we defined $R(u^\indexL_k)=0$ for $u^\indexL_k=0$. As $H$ is a norm, the equivalence of norm in finite dimensions implies that ${\normd{u^\indexL_k}}{H(u^\indexL_k)}$ is bounded by some constant $c<\infty$. We then have $\normd{{\bf u}_{k+1/2}}\leq\kappa= 1+\delta t\const_J\left(1+ \valL \const_H c\right)$.

\bigskip  Since ${\bf u}_{k+1/2}$ is the global minimiser of (10), then:
{\small{
\begin{equation*}
\begin{split}
&\hspace{3cm}F_k^\valL({\bf u}_{k+1/2})\leq F_k^\valL({\bf u}_k)\\
&\hspace{-0.3cm}\frac1{2\delta t}\normd{{\bf u}_{k+1/2}-{\bf u}_k}^2+\sum_{\indexL=1}^\valL J(u^\indexL_{k+1/2})\leq\sum_{\indexL=1}^\valL R(u_k^\indexL)\langle q_{k}^\indexL-\tilde q_k^\indexL,u^\indexL_{k+1/2}\rangle\\
&\hspace{-0.3cm}\frac1{2\delta t}\normd{{\bf u}_{k+1/2}-{\bf u}_k}^2+\sum_{\indexL=1}^\valL J(u^\indexL_{k+1/2})\leq  \sum_{\indexL=1}^\valL R(u_k^\indexL)H(u_{k+1/2}^\indexL)\\
&\hspace{-0.3cm}\sum_{\indexL=1}^\valL\left( J(u^\indexL_{k+1/2}) - \frac{J(u_k^\indexL)}{H(u_k^\indexL)}H(u_{k+1/2}^\indexL)  \right)\leq  -\frac1{2\delta t}\normd{{\bf u}_{k+1/2}-{\bf u}_k}^2\\
&\hspace{-0.3cm}\normd{{\bf u}_{k+1/2}} \sum_{\indexL=1}^\valL H(u_{k+1}^\indexL) \left( R(u_{k+1}^\indexL) - R(u_k^\indexL)  \right)\leq  -\frac1{2\delta t}\normd{{\bf u}_{k+1}-{\bf u}_k}^2\\
&\hspace{-0.3cm}\sum_{\indexL=1}^\valL H(u_{k+1}^\indexL) \left( R(u_{k+1}^\indexL) - R(u_k^\indexL)  \right)\leq  -\frac1{2\delta t\kappa}\normd{{\bf u}_{k+1}-{\bf u}_k}^2. \qedhere
\end{split}
\end{equation*}
}}
\end{proof}


Point 3 of Proposition \ref{prop2} contains weights $H(u_{k+1}^\indexL)$ that prevent from showing the exact decrease of the sum of ratios. This is in line with the results in \cite{bresson2013multiclass}.
To ensure the decrease of the sum of ratios
$\sum_{\indexL=1}^\valL {J(u^\indexL_k)}/{ H(u^\indexL_k)}$, it is possible to introduce auxiliary variables to deal with individual ratio decrease, as in \cite{rangapuram2014tight}. The involved sub-problem at each iteration $k$ is nevertheless more complex to solve.


\subsection{Introducing Label Priors}
The partitioning process induced by the scheme \eqref{pde2}
so far does not integrate any label information.
As we are working in a semi-supervised setting, we consider given a small subsets of labelled nodes $\mathcal{N}^\indexL\subset\mathcal{N}$ (with $|\mathcal{N}^\indexL|<<|\subset\mathcal{N}|$) belonging to each cluster $i$, with $\mathcal{N}^\indexL\cap \mathcal{N}_{j}=\emptyset$.
 Denoting as  $\mathcal{L}=\cup_{\indexL=1}^\valL \mathcal{N}^\indexL$, the objective is to propagate the prior information in the graph in order to infer pseudo-labels for the remaining nodes $x\in\mathcal{N}\backslash\mathcal{L}$.
To that end, we simply have to modify the coupling constraint $C$ in \eqref{constr_coupling} as
\begin{equation}
    \label{constr_coupling2}
C:\left\{{\bf u},\, \textrm{s.t. }\begin{array}{ll}
  \sum_{\indexL=1}^\valL u^\indexL(x)=0& \textrm{if }x\in\mathcal{N}\backslash \mathcal{L}\\
  u^\indexL(x)\geq \epsilon&\textrm{if }x\in\mathcal{N}^\indexL\\
  u^{\indexL'}(x)\leq -\epsilon, \forall \indexL'\neq \indexL&\textrm{if }x\in  \mathcal{L}\backslash \mathcal{N}^\indexL
  \end{array}
  \right\}.
  \end{equation}

With such constraint, clusters can no longer vanish or merge since they all contain different active nodes $x\in\mathcal{N}^\indexL$ satisfying $u^\indexL(x)>0$.
The same scheme \eqref{pde2} with the new constraint set \eqref{constr_coupling2} can be applied to propagate these labels. Once it has converged, the inferred pseudo-label of each unlabelled node $x\in\mathcal{N}\backslash\mathcal{L}$ is taken as:
\begin{equation}\label{final_out}
L(x)\in\uargmax{i\in \{1,\cdots \valL\}} u^\indexL(x).\end{equation}
Soft labelling can either be obtained by considering all the clusters with non negative weights $\mathcal{I}(x)=\{\indexL,\, u^\indexL(x)\geq 0\}\neq \emptyset$ and with relative weights $w^\indexL(x)=u^\indexL(x)/(\sum_{\indexL\in\mathcal{I}(x)}u^\indexL(x))$, with the convention that $w^\indexL(x)=1/\valL$, in the case that $u^\indexL(x)=0$ for all $\indexL=1\cdots\valL$ (which has never been observed in our experiments). For notation purposes and following the notation in~\eqref{generalSSL}, we denote the output of~\eqref{final_out} as $\hat{y_i}$, having $\hat{Y}=\{\hat{y}_k\}_{k=l+1}^{n}$.
The parameter $\epsilon$ in \eqref{constr_coupling2} is set to a small numerical value.

\subsection{Hybrid Framework}~\label{section:hybrid}
We now fit our energy model into a hybrid framework as displayed in Fig.~\ref{fig:hybridFramework}. We perform an alternating optimisation between \eqref{generalSSL} and  \eqref{pde2}. The process \eqref{pde2}, that provides pseudo-labels $\hat y$, is extensively described in previous subsections (as our main contribution). The  functional  \eqref{generalSSL} is now detailed. For the first term in \eqref{generalSSL}, we use a cross entropy loss with a weighting parameter as imbalance class strategy. We follow a standard strategy  e.g.~\cite{he2013imbalanced,bookImbalanceHerrera} such that the parameter is inversely proportional to the number of samples for class $k: \beta_k \propto 1/\mathcal{E}_n$ where $\mathcal{E}_n$ is the total number of samples. The second term in \eqref{generalSSL}  involves the inferred pseudo-labels updates $\hat y$ in a cross entropy loss, along with a dual weighting parameter $\beta\varphi$ where $\varphi$ is a measure of the uncertainty referring to the entropy. The remaining of the experiments follows this alternating optimisation. The choice for $f_\theta$ is discussed in the experimental results section.

\section{Experimental Results}
This section  focuses on the detailed description of the experiments that we conducted to evaluate our proposed approach.

\subsection{Data Description}
We extensively evaluate our approach using six very diverse datasets.
Firstly, we use the Fashion-MNIST~\cite{xiao2017fashion} dataset.  The dataset is composed of 70k grayscale images containing 10 classes from fashion items.  To further support the generalisation and robustness of our technique, we  use two major complex datasets from the medical domain. The  ChestX-ray14 dataset~\cite{wang2017chestx}  is composed of 112,120 frontal chest view X-ray with size of 1024$\times$1024 and 14 classes reflecting diverse pathologies. The  CBIS-DDSM dataset~\cite{lee2017curated}, composed of  3,103 mammography images with a mean size  of 3138$\times$5220, contains normal, benign, and malignant cases with verified pathology information.   Finally, we use three natural image datasets. The CIFAR-10 and CIFAR-100 dataset~\cite{krizhevsky2009learning} that contains 60k colour images of size 32$\times$32 with 10 and 100 different classes respectively. Finally, the Mini-ImageNet~\cite{vinyals2016matching} dataset consisting of 60k colour images with of size $64\times64$ and 100 classes.

\begin{table*}[t!]
\centering
\caption{Numerical comparison of our approach vs other energy-based approaches. The values are computed from 20 runs as the average of error rate $\pm$ standard deviation over several label counts.  The best results are highlighted  \colorbox{colorbrewer3!20}{green} whilst the second best ones are highlighted in  \colorbox{colorbrewer1!10}{red}.}
\renewcommand{\arraystretch}{1.1}
\begin{tabular}{ccccccc}
\hline
\rowcolor[HTML]{EFEFEF}
\cellcolor[HTML]{EFEFEF} & \multicolumn{6}{c}{\cellcolor[HTML]{EFEFEF}\% \textsc{Labelled Set} } \\ \cline{2-7}
\rowcolor[HTML]{EFEFEF}
\multirow{-2}{*}{\cellcolor[HTML]{EFEFEF}\textsc{Technique}} & 1\% & 2\% & 5\% & 10\% & 20\% & 30\% \\ \hline \hline
Harmonic Gaussian (HG)~\cite{zhu2002learning} & 18.97$\pm$0.34
 & 17.96$\pm$0.23 & \cellcolor{red!6}16.35$\pm$0.10  & 15.33$\pm$0.10 & 14.48$\pm$0.11 & 14.09$\pm$0.08 \\ \hspace*{2px}
Local to Global Consistency (LGC)~\cite{zhou2004learning} & \cellcolor{red!6}18.65$\pm$0.59 & \cellcolor{red!6} 17.81$\pm$0.31 & 16.41$\pm$0.12 & \cellcolor{red!6}15.20$\pm$0.15 &\cellcolor{red!6} 14.33$\pm$0.13 &  \cellcolor{red!6}13.68$\pm$0.06\\
Lazy Random Walks (LRW)~\cite{zhou2004learningb} & 19.09$\pm$0.25 & 17.84$\pm$0.17 & 16.38$\pm$0.06 & 15.69$\pm$0.10 & 15.35$\pm$0.10 & 15.19$\pm$0.07 \\
Sparse Label Propagation (SLP)~\cite{jung2016semi}  & 78.18$\pm$1.37 & 50.16$\pm$3.81 & 44.18$\pm$0.37 & 25.17$\pm$1.75 & 19.03$\pm$1.75 &  13.16$\pm$0.11\\
Weighted Nonlocal Laplacian (WNLL)~\cite{shi2017weighted} & 19.68$\pm$0.17 & 18.96$\pm$0.28 & 17.41$\pm$0.12 & 16.15$\pm$0.17 & 14.93$\pm$0.13 & 14.32$\pm$0.09 \\
Centered Kernel (CK)~\cite{mai2018random,mai2021consistent} & 31.16$\pm$1.28 & 24.29$\pm$0.61 & 20.30$\pm$0.11 & 18.62$\pm$0.24 & 16.54$\pm$0.12 & 15.38$\pm$0.07 \\
Poisson Learning (PoL)~\cite{calder2020poisson} & 20.60$\pm$1.37  & 19.89$\pm$3.81 & 19.17$\pm$0.37  & 19.03$\pm$1.75 & 18.82$\pm$0.71 & 18.87$\pm$0.41 \\
\rowcolor{pinegreen!12}
\cellcolor{white} Ours & \textbf{17.62$\pm$0.21}  & \textbf{16.05$\pm$0.24} & \textbf{14.21$\pm$0.08} & \textbf{13.16$\pm$0.11} & \textbf{12.31$\pm$0.09}  & \textbf{11.82$\pm$0.04} \\ \hline
\end{tabular}
\label{table::energyModels}
\end{table*}

\subsection{Evaluation Protocol}
We design the following evaluation scheme to validate our theory.

\textbf{Baseline Comparison against existing energy models.} As the core contribution of this work is a new graph based energy model, we first compared our technique against existing graph-based energy methods: Harmonic Gaussian (HG)~\cite{zhu2002learning}, Local to Global Consistency (LGC)~\cite{zhou2004learning}, Lazy Random Walks (LRW)~\cite{zhou2004learningb}, Sparse Label Propagation (SLP)~\cite{jung2016semi}, Weighted Nonlocal Laplacian (WNLL)~\cite{shi2017weighted}, Centered Kernel (CK)~\cite{mai2018random,mai2021consistent} and Poisson Learning (PoL)~\cite{calder2020poisson}. To solely evaluate the impact of these energy models, we used the same network architecture and substituted each method as the energy model  (i.e. only green box of Fig.~\ref{fig:hybridFramework}). We run the experiments for all techniques under the same conditions by constructing a k-NN graph with $k=20$ using the features extracted from a 13-Layer Network. We remark that these set of experiments are to purely compare energy models and not networks Fig.~\ref{fig:hybridFramework}. We use different label \% counts \{1,2,5,10,20,30\}, and  report the mean error and standard deviation  over  randomly select the labelled samples over twenty repeated times (20 different splits).

\textbf{Comparison against SOTA Techniques.} For our full model we compare to the state-of-the-art for each of the differing domains

\textit{Medical Datasets.} For the ChestX-ray14 dataset~\cite{wang2017chestx},  we firstly compared against the SOTA supervised techniques of \cite{wang2017chestx,yao2018weakly,guendel2018learning,shen2018dynamic,baltruschat2019comparison,rajpurkar2017chexnet,guan2020multi,ma2019multi,kim2021xprotonet} using the official partition of the dataset (70\% labelled data) against ours using 20\% of labelled data. Moreover, we compared against the  SOTA semi-supervised techniques of \cite{tarvainen2017mean,aviles2019graphx,liu2020semi,liu2021self}. All  semi-supervised techniques are reported using 20\% of labels. We also provide comparison with existing  techniques~\cite{lee2017curated} \cite{zhu2017deep,wei2021beyond,shu2020deep,shen2021interpretable} on the CBIS-DDSM dataset.
The quality check is performed following standard convention in the medical domain by a ROC analysis using the area under the curve (AUC).

\textit{Natural Image Datasets.} Finally, we report results against the SOTA semi-supervised techniques for natural image datasets: $\Pi-$Model and Temporal Ensembling~\cite{laine2016temporal}, Mean Teacher (MT)~\cite{tarvainen2017mean}, VAT~\cite{miyato2018virtual}, $\Pi+$SNTG~\cite{luo2018smooth}, MT+fast-SWA~\cite{athiwaratkun2018there}, MT+ICT~\cite{verma2019interpolation}, Dual Student~\cite{ke2019dual}, MUSCLE+MT+LP~\cite{xie2021muscle}, MT+TSSDL~\cite{shi2018transductive}, MT+LP~\cite{iscen2019label}, CycleCluster~\cite{sellars2020two}, DAG~\cite{li2020density}, UPS~\cite{rizve2020defense}, PL-Mixup~\cite{arazo2020pseudo}, LaplaceNet~\cite{sellars2021laplacenet}, UDA~\cite{xie2020unsupervised},  SimPLE~\cite{hu2021simple}, FixMatch~\cite{sohn2020fixmatch}. We evaluate the quality of the classifiers by reporting the error rate and standard deviation over five runs and  for a range number of labelled samples.

\begin{figure}[t!]
    \centering
    \includegraphics[width=0.45\textwidth]{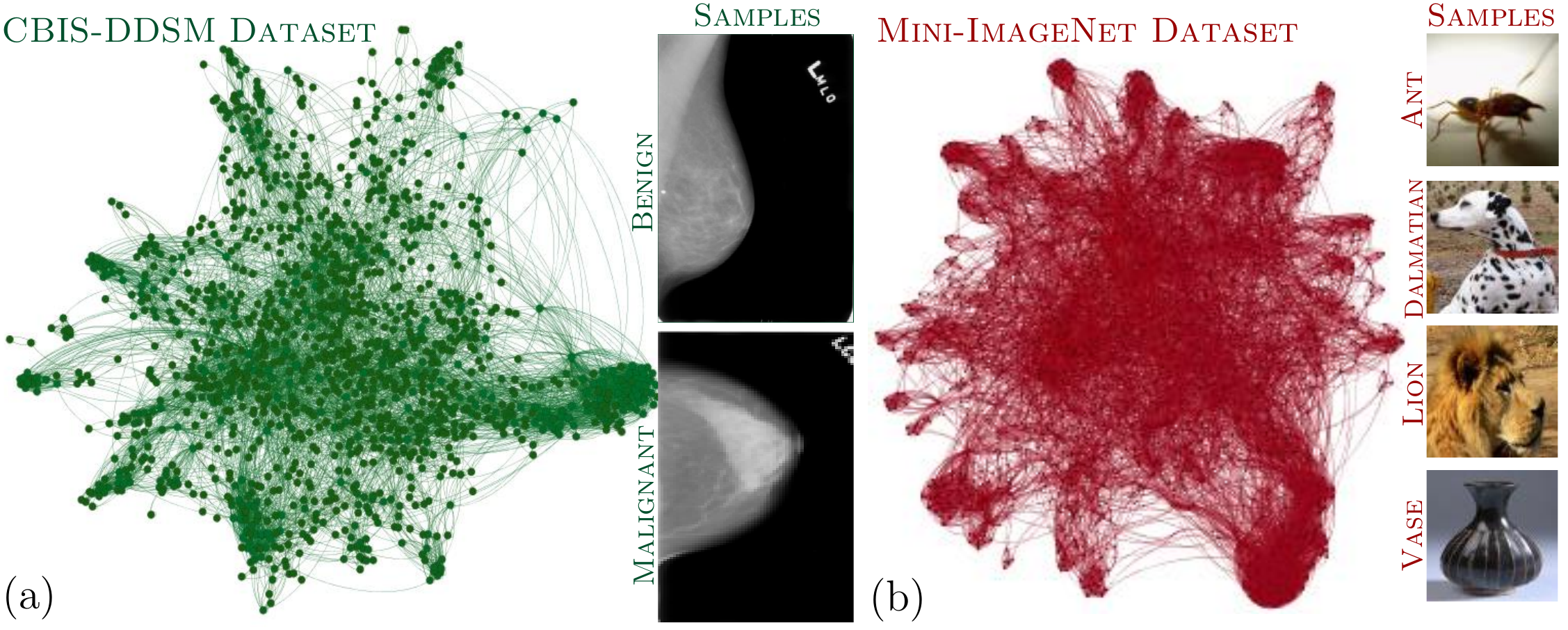}
    \caption{Graph visualisation for two selected datasets used in our experiments. Graphical representation for (a) the CBIS-DDSM dataset and (b) the Mini-ImageNet Dataset. A few sample images for each datasets are displayed at the left side of each graph. }
    \label{fig:graphVis}
\end{figure}

\subsection{Implementation Details}
We set the architecture for  $f_\theta$ (i.e. red box of Fig.~\ref{fig:hybridFramework}) as follows. For the medical datasets we use a ResNet-18~\cite{he2016deep}. For the natural image datasets, we ran
experiments with three different networks. For CIFAR-10 and CIFAR-100, we divided our experiments into two parts. For the first part,  we use a 13-Layer Network for a fair comparison as existing approaches run under this architecture. For the second part and motivated by the work of~\cite{oliver2018realistic}, we compare against the most recent techniques under exactly the same conditions which includes the optimiser, RandAugment implementation and a WideResNet-28-2 (WRN-28-2). Finally,  we use a ResNet-18 for the Mini-Imagenet dataset as fair comparison for existing techniques. For the graph generation, a k-NN graph with $k=20$ is constructed  using the features from each respective architecture - Fig.~\ref{fig:graphVis} displays  examples of generated graphs for two selected datasets.
For our approach, we set the number of epochs of 310 and a weight decay of $2\times10^{-4}$. The learning rate is set to 5e-2 and with a scheduled cosine annealing. We use as optimiser stochastic gradient descent (SGD) and  implement our code in PyTorch.

\begin{table*}[t!]
\caption{{Performance comparison of our approach (20\% of labelled data) against SOTA fully-supervised techniques (70\% of labelled data). We reported the AUC per class and average AUC  over all 14 pathologies. All compared techniques are performed using the official data split.
The results in bold denotes the highest performance.}}
\label{table::ChestXrayfull}
\begin{tabular}{lccccccccc|cc}
\cline{2-11}
\textsc{ChestX-ray14} & \multicolumn{9}{c}{\cellcolor[HTML]{EFEFEF}\textsc{Fully Supervised Techniques (70\% Labelled data)}} & \multicolumn{1}{|c}{\cellcolor[HTML]{EFEFEF}SSL } \\ \hline
\multicolumn{1}{c}{\textsc{Pathology}} & \begin{tabular}[c]{@{}c@{}}{Wang et al}\\ {\cite{wang2017chestx}}\end{tabular} & \begin{tabular}[c]{@{}c@{}}{Yao et al}\\ {\cite{yao2018weakly}}\end{tabular} & \begin{tabular}[c]{@{}c@{}}{Guendel}\\ {et al\cite{guendel2018learning}}\end{tabular}& \begin{tabular}[c]{@{}c@{}}{Shen et al}\\ {\cite{shen2018dynamic}}\end{tabular}  &
\begin{tabular}[c]{@{}c@{}}Baltruschat, \\ {\cite{baltruschat2019comparison}}\end{tabular} & \begin{tabular}[c]{@{}c@{}}{CheXNet}\\ {\cite{rajpurkar2017chexnet}}\end{tabular} & \begin{tabular}[c]{@{}c@{}}{Guan}\\ et al\cite{guan2020multi}\end{tabular}& \begin{tabular}[c]{@{}c@{}}{Ma et al}\\  \cite{ma2019multi}\end{tabular} &
\begin{tabular}[c]{@{}c@{}}Kim, \\ et al\cite{kim2021xprotonet}\end{tabular} & \quad \begin{tabular}[c]{@{}c@{}}CREPE \\ (Ours)\end{tabular} \quad\\ \hline \hline
Atelectasis &  70.03 & 73.30 & 76.70 & 76.60 &  76.30 & 75.9 & 78.10 & 77.70 & 78.2 & 78.65\\
Cardiomegaly & 81.00 & 85.60 & 88.30 & 80.10 &  87.50 & 87.1 & 88.30 & 89.40 & 88.1 & 88.74\\
Effusion & 75.85 & 80.60 & 82.80 & 79.70 &  82.20 & 82.1 & 83.10 & 82.90 & 83.6 & 83.15 \\
Infiltration & 66.14 & 67.30 & 70.90 & 75.10 & 69.40 & 70.0 & 69.70 & 69.60 & 71.5 & 72.25  \\
Mass & 69.33 & 77.70 & 82.10 & 76.00 &  82.00 & 81.0 & 83.00 & 83.80 & 83.4 & 83.41\\
Nodule & 66.87 & 71.80 & 75.80 & 74.10 &  74.70 & 75.9 & 76.40 & 77.71 & 79.9 & 76.61\\
Pneumonia & 65.80 & 68.40 & 73.10 & 77.80 &   71.40 & 71.8 & 72.50 & 72.20 & 73.0 & 76.04\\
Pneumothorax & 79.93 & 80.50 & 84.60 & 80.00 & 84.00 & 84.8 & 86.60 & 86.20 & 87.4 & 86.89 \\
Consolidation & 70.32 & 71.10 & 74.50 & 78.70 & 74.90 & 74.1 & 75.80 & 75.00 & 74.7  & 75.42 \\
Edema & 80.52 & 80.60 & 83.50 & 82.00 &  84.60 & 84.4 & 85.30 & 84.60 & 83.4  & 84.96\\
Emphysema & 83.30 & 84.20 & 89.50 & 77.30 &  89.50 & 89.1 & 91.10 & 90.80 & 93.6  & 90.95\\
Fibrosis & 78.59 & 74.30 & 81.80 & 76.50 &  81.60 & 81.0 & 82.60 & 82.70 & 81.5 & 82.16\\
Pleural Thicken & 68.35 & 72.40 & 76.10 & 75.90 &   76.30 & 76.8 & 78.00 & 77.90 & 79.8 & 76.84\\
Hernia & 87.17 & 77.50 & 89.60 & 74.80 &  93.70 & 86.7 & 91.80 & 93.40 & 89.6 & 88.38\\ \hline
\textsc{Average AUC} & 74.51 & 76.09 & 80.66 & 77.47 &  80.57 & 80.05 & 81.60 & 81.71 & \textbf{82.0} & \textbf{81.75}\\ \hline
\end{tabular}
\end{table*}
\begin{table*}[t!]
\begin{minipage}[t]{.62\textwidth}
\caption{Numerical comparison of our technique and existing semi-supervised approaches for the ChestX-ray14 dataset. All techniques use 20\%of labelled data, and the reported results reflect the AUC per class and average.  The best results are highlighted  \colorbox{colorbrewer3!20}{green}.}
\label{table::SSLmedical1}
\begin{tabular}{lcccccc}
\cline{2-7}
\textsc{ChestX-ray14} & \multicolumn{6}{c}{\cellcolor[HTML]{EFEFEF}{\color[HTML]{000000} \textsc{Semi-Supervised/Self-Supervised Techniques }}} \\ \hline
\multicolumn{1}{c}{\textsc{Pathology}} & \begin{tabular}[c]{@{}c@{}}MT\\ {\cite{tarvainen2017mean}}\end{tabular} & \begin{tabular}[c]{@{}c@{}}GraphXNet\\ {\cite{aviles2019graphx}}\end{tabular} & \begin{tabular}[c]{@{}c@{}}MOCOV2\\ {\cite{chen2020improved,liu2021self}}\end{tabular} &  \begin{tabular}[c]{@{}c@{}}SRC-MT\\ {\cite{liu2020semi}}\end{tabular} &  \begin{tabular}[c]{@{}c@{}}S$^2$MTS$^2$\\ {\cite{liu2021self}}\end{tabular} &  {\quad \begin{tabular}[c]{@{}c@{}}CREPE \\ (Ours)\end{tabular}  \quad }\\ \hline \hline
Atelectasis & 75.12 & 71.89 & 77.21  & 75.38 & 78.57 &     \cellcolor{pinegreen!12} \textbf{78.65} \\
Cardiomegaly & 87.37 & 87.99 & 85.84 & 87.70 & 88.08 &  \cellcolor{pinegreen!12}\textbf{88.74} \\
Effusion & 80.81 & 79.20 & 81.62 & 81.58 & 82.87 &  \cellcolor{pinegreen!12}\textbf{83.15} \\
Infiltration & 70.67 & 72.05 & 70.91 & 70.40 &  70.68 & \cellcolor{pinegreen!12}\textbf{72.25} \\
Mass & 77.72 & 80.90 & 81.71 & 78.03 & 82.57 &  \cellcolor{pinegreen!12}\textbf{83.41} \\
Nodule & 73.27 & 71.13 & \cellcolor{pinegreen!12}\textbf{76.72} & 73.64 & 76.60 &  76.61 \\
Pneumonia & 69.17 & \cellcolor{pinegreen!12}\textbf{76.64} & 71.08 & 69.27 & 72.25 &  76.04 \\
Pneumothorax & 85.63 & 83.70 & 85.92 & 86.12 & 86.55 &  \cellcolor{pinegreen!12}\textbf{86.89} \\
Consolidation & 72.51 & 73.36 & 74.47 & 73.11 & \cellcolor{pinegreen!12}\textbf{75.47}  &  75.42\\
Edema & 82.72 & 80.20 & 83.57 & 82.94 &  84.83 & \cellcolor{pinegreen!12}\textbf{84.96} \\
Emphysema & 88.16 & 84.07 & 91.10 & 88.98 & \cellcolor{pinegreen!12}\textbf{91.88} &  90.95 \\
Fibrosis & 78.24 & 80.34 & 80.96 & 79.22 & 81.73 &  \cellcolor{pinegreen!12}\textbf{82.16} \\
Pleural Thicken & 74.43 & 75.70 & 75.65 & 75.63 & \cellcolor{pinegreen!12}\textbf{76.86} &  76.84 \\
Hernia & 87.74 & 87.22 & 85.62 & 87.27 &  85.98 &  \cellcolor{pinegreen!12}\textbf{88.38}\\ \hline
\textsc{Average AUC} & 78.83 & 78.88 & 80.17 & 79.23 & 81.06 &  \cellcolor{pinegreen!12}\textbf{81.75} \\ \hline
\end{tabular}
\end{minipage}
\hspace{0.7cm}
\begin{minipage}[t]{.30\textwidth}
\begin{flushleft}
\caption{ AUC performance comparison of existing SOTA supervised techniques and our technique for the CBIS-DDSM dataset. The best results are marked in  \colorbox{colorbrewer3!20}{green} whilst the second best in \colorbox{colorbrewer1!10}{{red}}.}
\label{table::SSLmedical2}
\begin{tabular}{lcccc}
\hline
\multicolumn{4}{c}{\cellcolor[HTML]{EFEFEF} \textsc{CBIS-DDSM Dataset}} \\ \hline
\cellcolor[HTML]{EFEFEF} & \multicolumn{2}{c}{ \cellcolor[HTML]{EFEFEF}\textsc{Paradigm}} & \cellcolor[HTML]{EFEFEF} \\
\multirow{-2}{*}{\cellcolor[HTML]{EFEFEF}\textsc{Technique}} & \multicolumn{1}{c}{\cellcolor[HTML]{EFEFEF}SL (85\%)} & \multicolumn{1}{c}{\cellcolor[HTML]{EFEFEF}SSL} & \cellcolor[HTML]{EFEFEF}\multirow{-2}{*}{\cellcolor[HTML]{EFEFEF}AUC} \\ \hline \hline
ResNet-34 & \checkmark &  & 79.2 \\
Zhu et al \cite{zhu2017deep} & \checkmark &  & 79.1 \\
Tao et al \cite{wei2021beyond} & \checkmark &  & 83.1 \\
Shu et al \cite{shu2020deep}  & \checkmark &  & 83.8 \\
Shen et al$\dagger$ \cite{shen2021interpretable}  & \checkmark &  & \cellcolor{red!6}84.0 \\   
CREPE  (Ours, 35\%) &  & \checkmark &  83.9\\
CREPE  (Ours, 40\%) &  & \checkmark &  \cellcolor{pinegreen!12}\textbf{84.2}\\ \hline
\end{tabular}
\end{flushleft}
\end{minipage}
\end{table*}

\begin{table*}[t!]
\caption{Comparison performance of our technique against existing semi-supervised techniques (\colorbox{blue!5}{consistency regularisation} and \colorbox{yellow!15}{pseudo-labelling} family of techniques) along with the fully supervised baseline for CIFAR-10 and CIFAR-100. All the results are derived from using a 13-Layer architecture, and reflect the error rate and standard deviation. The results in bold denotes the best performance. }
\centering
\renewcommand{\arraystretch}{1.1}
\label{table::13Layer}
\begin{tabular}{ccccccc}
\hline
& \multicolumn{3}{c}{\cellcolor[HTML]{EFEFEF}\begin{tabular}[c]{@{}c@{}}\textsc{Labelled Samples}\\ (CIFAR-10)\end{tabular}} &  & \multicolumn{2}{c}{\cellcolor[HTML]{EFEFEF}\begin{tabular}[c]{@{}c@{}}\textsc{Labelled Samples}\\ (CIFAR-100)\end{tabular}} \\ \cline{2-7}
\multirow{-2}{*}{\begin{tabular}[c]{@{}c@{}}\textsc{Technique}\\ (13-CNN)\end{tabular}} & 1k & 2k & 4k &  & 4k & 10k \\ \hline
Fully Supervised & 26.60$\pm$0.22 & 19.53$\pm$0.12 & 14.02$\pm$0.10 &  & 53.10$\pm$0.34 &  36.59$\pm$0.47\\ \hline \hline
\rowcolor{blue!3}
\multicolumn{7}{c}{\textsc{Consistency Regularisation Techniques}}  \\ \hline
\rowcolor{blue!3}
 $\Pi-$Model~\cite{laine2016temporal} & 31.65$\pm$1.20 & 17.57$\pm$0.44 & 12.36$\pm$0.31 &  & $-$ &  39.19$\pm$0.36\\
\rowcolor{blue!3}
Temporal Ensembling~\cite{laine2016temporal} & 23.31$\pm$1.01 & 15.64$\pm$0.39  & 12.16$\pm$0.24 &  & $-$ & 38.65$\pm$0.51 \\
\rowcolor{blue!3}
Mean Teacher (MT)~\cite{tarvainen2017mean}  & 21.55$\pm$1.48 & 15.73$\pm$0.31 & 12.31$\pm$0.28 &  & 45.36$\pm$0.49 & 36.08$\pm$0.51 \\
\rowcolor{blue!3}
VAT~\cite{miyato2018virtual}& $-$ & $-$ & 11.36$\pm$0.34 &  & $-$  &  $-$ \\
\rowcolor{blue!3}
 $\Pi+$SNTG~\cite{luo2018smooth}  & 21.23$\pm$1.27 & 14.65$\pm$0.31 & 11.00$\pm$0.13 &  & $-$ & 37.97$\pm$0.29 \\
\rowcolor{blue!3}
MT+fast-SWA~\cite{athiwaratkun2018there}  & 15.58$\pm$0.12 &  11.02$\pm$0.23 & 9.05$\pm$0.21  &  & $-$ &33.62$\pm$0.54  \\
\rowcolor{blue!3}
MT+ICT~\cite{verma2019interpolation}& 15.48$\pm$0.78  & 9.26$\pm$0.09  & 7.29$\pm$0.02 &  & $-$ & $-$ \\
\rowcolor{blue!3}
Dual Student~\cite{ke2019dual} & 14.17$\pm$0.38 & 10.72$\pm$0.19 & 8.89$\pm$0.09 &  &  $-$ & 32.77$\pm$0.24 \\
\rowcolor{blue!3}
MUSCLE+MT+LP~\cite{xie2021muscle} & 13.29$\pm$0.36 & $-$ & $-$ &  & 42.34$\pm$0.45 &  35.21$\pm$0.25\\ \hline \hline
\rowcolor{yellow!5}
\multicolumn{7}{c}{\textsc{Pseudo-Labelling Techniques}}  \\ \hline
\rowcolor{yellow!5}
MT+TSSDL~\cite{shi2018transductive} & 18.41$\pm$0.92 & 13.54$\pm$0.32 & 9.30$\pm$0.55 &  &  $-$  &  $-$  \\
\rowcolor{yellow!5}
MT+LP~\cite{iscen2019label} & 16.93$\pm$0.70 & 13.22$\pm$0.29 & 10.61$\pm$0.28 &  & 43.73$\pm$0.20 &  35.92$\pm$0.47\\
\rowcolor{yellow!5}
CycleCluster~\cite{sellars2020two} & 15.52$\pm$0.88 & 12.79$\pm$0.35 & 10.79$\pm$0.45 &  & 45.19$\pm$0.34  &  35.65$\pm$0.50\\
\rowcolor{yellow!5}
DAG~\cite{li2020density} & 7.42$\pm$0.41 & 7.16$\pm$0.38 & 6.13$\pm$0.15 &  & 37.38$\pm$0.64 &  32.50$\pm$0.21\\
\rowcolor{yellow!5}
UPS~\cite{rizve2020defense} & 8.18$\pm$0.15 & $-$ & 6.39$\pm$0.02 &  & 40.77$\pm$0.10 & 32.00$\pm$0.49 \\
\rowcolor{yellow!5}
PL-Mixup~\cite{arazo2020pseudo} & 6.85$\pm$0.15 & $-$ & 5.97$\pm$0.15 &  & 37.55$\pm$1.09 & 32.15$\pm$0.50 \\
\rowcolor{yellow!5}
LaplaceNet~\cite{sellars2021laplacenet} & 5.33$\pm$0.02 & 4.99$\pm$0.12& 4.64$\pm$0.07 &  & 31.64$\pm$0.02 & 26.60$\pm$0.23 \\
\rowcolor{yellow!5}
CREPE (Ours) & \textbf{5.04$\pm$0.03} & \textbf{4.58$\pm$0.11} & \textbf{4.31$\pm$0.08} &  & \textbf{31.02$\pm$0.03} & \textbf{25.11$\pm$0.19}  \\ \hline
\end{tabular}
\end{table*}

\begin{table*}[t!]
\centering
\begin{minipage}[b]{0.50\linewidth}
\caption{Performance comparison using CIFAR-10/100 reporting error rate$\pm$ standard deviation. The \colorbox{blue!5}{consistency regularisation} and \colorbox{yellow!15}{pseudo-labelling} techniques were run under the same code-base using same architecture WRN-28-2. }
\centering
\label{table::WRN}
\begin{tabular}{cccccc}
\hline
& \multicolumn{2}{c}{\cellcolor[HTML]{EFEFEF}\begin{tabular}[c]{@{}c@{}}Labelled Samples\\ (CIFAR-10)\end{tabular}} &  & \multicolumn{2}{c}{\cellcolor[HTML]{EFEFEF}\begin{tabular}[c]{@{}c@{}}Labelled Samples\\ (CIFAR-100)\end{tabular}} \\ \cline{2-6}
\multirow{-2}{*}{\begin{tabular}[c]{@{}c@{}}\textsc{Technique}\\ (WRN-28-2)\end{tabular}} & 2k & 4k &  & 4k & 10k \\ \hline \hline
\rowcolor{blue!3}
UDA~\cite{xie2020unsupervised} & 5.61$\pm$0.16 & 5.40$\pm$0.19 &  & 36.19$\pm$0.39 & 31.49$\pm$0.19 \\
\rowcolor{blue!3}
FixMatch~\cite{sohn2020fixmatch} & 5.42$\pm$0.11 & 5.30$\pm$0.08 &  & 34.87$\pm$0.17 &  30.89$\pm$0.18\\
\rowcolor{yellow!5}
SimPLE~\cite{hu2021simple} & 5.27$\pm$0.18 & 5.33$\pm$0.20 &  & 34.75$\pm$0.16 & 29.18$\pm$0.25 \\
\rowcolor{yellow!5}
LaplaceNet~\cite{sellars2021laplacenet} & 4.71$\pm$0.05 & 4.35$\pm$0.10 &  & 33.16$\pm$0.22 & 27.49$\pm$0.22 \\
\rowcolor{yellow!5}
CREPE (Ours)  & \textbf{4.33$\pm$0.09} & \textbf{4.16$\pm$0.11} &  & \textbf{32.21$\pm$0.18} & \textbf{26.14$\pm$0.24} \\ \hline
\end{tabular}
\end{minipage}
\hspace{1cm}
\begin{minipage}[b]{0.4\linewidth}
\caption{Error rate ($\pm$ standard deviation) comparison for Mini-ImageNet dataset. All techniques use a ResNet-18 Network. Numbers in bold indicate best performance.}
\label{table::MiniImagenet}
\begin{tabular}{ccc}
\hline
 & \multicolumn{2}{c}{\cellcolor[HTML]{EFEFEF}\begin{tabular}[c]{@{}c@{}}\textsc{Labelled Samples}\\ \textsc{(Mini-ImageNet)}\end{tabular}} \\
\multirow{-2}{*}{\textsc{Technique}} & 4k & 10k \\ \hline
\rowcolor{blue!3}
Mean Teacher (MT)~\cite{tarvainen2017mean} & 72.51$\pm$0.22 &  57.55$\pm$1.11\\
\rowcolor{yellow!5}
LP~\cite{iscen2019label} & 70.29$\pm$0.81 & 57.58$\pm$1.47 \\
\rowcolor{yellow!5}
Two Cycle Learning~\cite{sellars2020two} & 69.12$\pm$1.05 & 54.27$\pm$0.71 \\
\rowcolor{yellow!5}
PL-Mixup~\cite{arazo2020pseudo} & 56.49$\pm$0.51 &  46.08$\pm$0.11\\
\rowcolor{yellow!5}
SimPLE~\cite{hu2021simple} & 50.21$\pm$0.42 & 43.44$\pm$0.12 \\
\rowcolor{yellow!5}
LaplaceNet~\cite{sellars2021laplacenet} & 46.32$\pm$0.27 &  39.43$\pm$0.09\\
\rowcolor{yellow!5}
CREPE (Ours)  & \textbf{{45.61$\pm$0.25}} & \textbf{38.33$\pm$0.11} \\ \hline
\end{tabular}
\end{minipage}
\end{table*}

\subsection{Results \& Discussion.}
In this section, we report and discuss the results and comparison of our proposed technique.

\medskip
\textbf{How Good is our Energy Model?} We start by evaluating the performance of our energy model.  To do this, we ran a set of comparisons of our technique against existing energy models including recent ones. For a fair comparison all techniques were fed with the same graph (constructed as detailed in previous subsection).
The results are displayed in Table~\ref{table::energyModels}, which reports the error rate averaged over 20 runs and the standard deviation under different \% of labelled data. In a closer look at the results, we can observe that our approach reports the lowest error rate for all label counts whilst LGC~\cite{zhou2004learning} ranked second. The techniques of CK~\cite{mai2018random} and SLP~\cite{jung2016semi} failed to be robust in the low label regime and needed a higher number of labels to improve the performance than the compared techniques. A similar  performance behaviour was observed in the techniques of HG~\cite{zhu2002learning}, LGC~\cite{zhou2004learning}, LRW~\cite{zhou2004learningb} and WNLL~\cite{shi2017weighted}. In contrast to the compared techniques, the performance of PoL~\cite{calder2020poisson} was not improved when more labels are considered.
Our technique reported a percentage of improvement in the range of 6\% to 14\% with respect to LGC, the second best ranked technique.
Overall and from the results, we highlight a \textit{key strength} of our energy approach  -- \textit{it demonstrates a good generalisation performance in the low regime labelled set, and consistent performance improvement when more labels are considered.}

\smallskip
\textbf{Hybrid Semi-Supervised Medical Image Classification.} We now evaluate our full hybrid framework (see Fig.~\ref{fig:hybridFramework}). We start by using the ChestX-ray14 benchmarking dataset~\cite{wang2017chestx}. We first compared our approach against the SOTA supervised techniques, they assume a large corpus of annotated data (70\%) whilst our technique reports the performance using 20\% of labels. The results are reported in Table~\ref{table::ChestXrayfull} displaying the AUC per class and the average over all classes. By inspection we can observe that our technique readily competes with existing deep supervised techniques in per class performance. Overall, our technique outperformed almost all existing techniques and places second behind the work of~\cite{kim2021xprotonet}. However, we remark that our technique is using far less labels (only 20\%) than all compared techniques (fully supervised 70\% of labelled data).

\begin{figure*}[t!]
    \centering
    \includegraphics[width=0.92\textwidth]{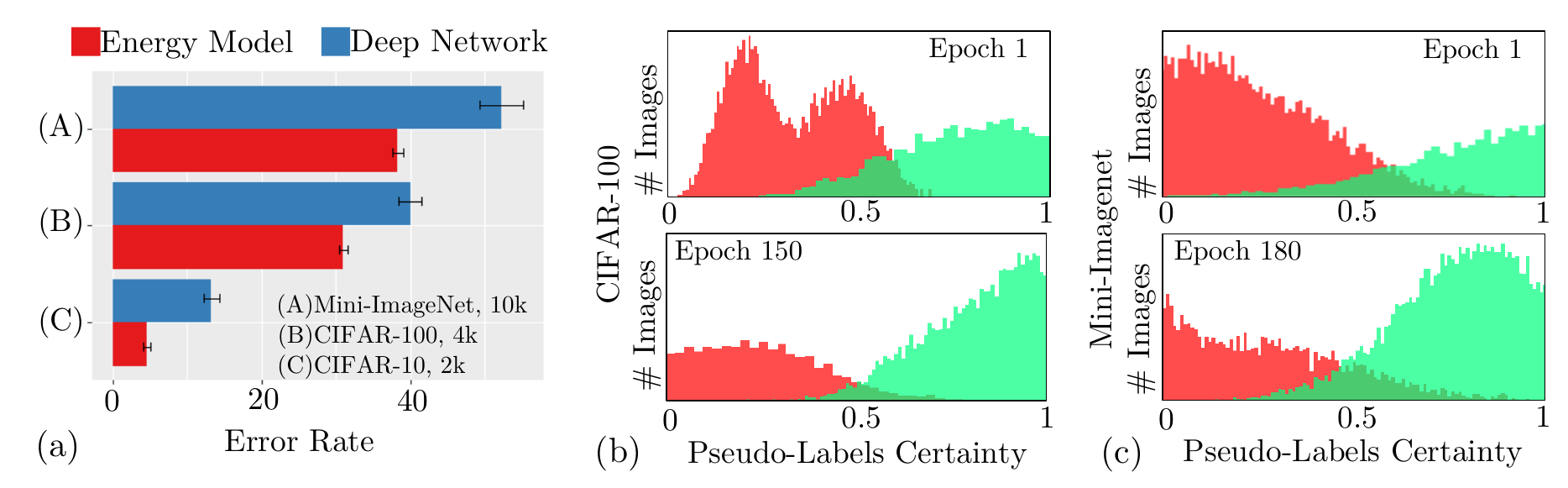}
    \caption{ (a) Error rate comparison for pseudo-label generation between our energy model vs the deep network for Mini-Imagenet, CIFAR-10/100.   (b)-(c) Certainty $\varphi$ of the pseudo-labels correctness (green) along with those incorrect (red) for two selected epochs for CIFAR-100 (4k labels) and Mini-Imagenet (10k labels) correspondingly.      }
    \label{fig:ablation1}
\end{figure*}

\begin{figure}[t!]
    \centering
    \includegraphics[width=0.5\textwidth]{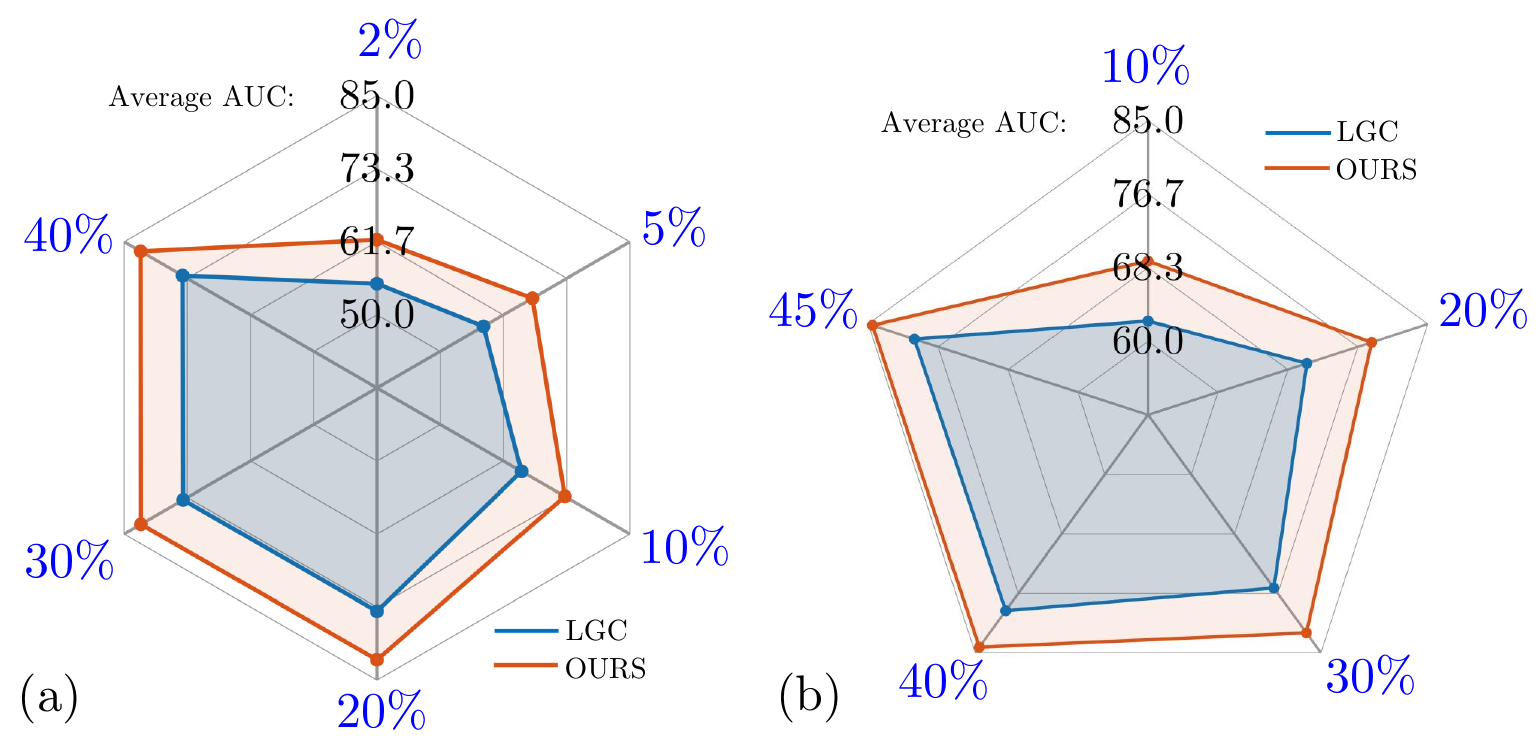}
    \caption{Average AUC comparison of hybrid models performance using  our energy model (red line) and that of~\cite{zhou2004learning} (blue line). (a) reports the comparison performance for different \% of labels for the ChestX-ray14 dataset whilst b) for the CBIS-DDSM dataset.  }
    \label{fig:ablation2}
\end{figure}

We also compared our proposed technique against existing semi-supervised models for medical imaging. The results are reported in Table~\ref{table::SSLmedical1}.
All results are produced using 20\% of labelled data, and the table displays the AUC per class and average over all classes. We observe that our technique reports the best AUC per class in majority of the pathologies, whith a best overall score (see scores highlighted in green).  We further support the performance of our method with another challenging medical dataset the CBIS-DDSM dataset~\cite{lee2017curated} for mammograms classification. We compare our approach against existing techniques for such dataset and report the results in Table~\ref{table::SSLmedical2}.
The compared techniques are deep supervised techniques, and to the best of our knowledge, there exists no modern semi-supervised techniques for this dataset to compare  with. For the supervised techniques the official partition is used (i.e. 85\% of labelled data). For our technique, we reported the AUC as result of an average of 5 split runs using 35\% and 40\% of labelled data. From the results, our technique produced readily compared performance whilst using only a fraction of labels, and reported the highest AUC using less than half of the labels than the compared techniques.

\textbf{Comparison with SOTA Semi-Supervised Models.} For our final set of experiments, we compared our technique against existing semi-supervised techniques for natural images where numerous methods have been proposed. For a fair comparison, we first provide a comparison using a 13-Layer Network, which is the most widely used network in semi-supervised classification. The results, in terms of error rate, are presented in Table~\ref{table::13Layer}. We observe that our technique provides a substantial improvement in performance with respect to consistency regularisation techniques for both CIFAR-10/100 datasets. In terms of existing pseudo-labelling techniques, our technique provides a significant margin of improvement. We namely obtain the lowest error rates for all the label counts and for both datasets.

Most recent and current SOTA techniques are based on more complex optimisation schemes scaling to more modern networks. Therefore, we also provide results against the techniques of: UDA~\cite{xie2020unsupervised}, FixMatch~\cite{sohn2020fixmatch}, SimPLE~\cite{hu2021simple} and LaplaceNet~\cite{sellars2021laplacenet}. To do this and following ~\cite{oliver2018realistic} for a fair comparison, we ran those set of techniques under the same code-based (i.e. the same implementation for the augmentations (RandAugment), optimiser and network architecture) using the same backbone a WRN-28-2. The performance comparison in terms of error rate is reported  in Table~\ref{table::WRN} using \{2k, 4k\} and \{4k, 10k\} labels for  CIFAR-10/100  respectively. Our technique reported the lowest error rate for all label counts and both datasets. We thus observe  a significant performance improvement with respect to consistency regularisation techniques for larger class number (CIFAR-100). Finally to further support the generalisation of our technique, we report results for Mini-ImageNet in Table~\ref{table::MiniImagenet}. In this experiment all methods use a ResNet-18. We highlight that for this complex dataset,  our technique reports a substantial performance improvement ([3\%, 37\%]).

\textbf{Our Energy Model vs Deep Network for Pseudo-labelling.} Our graph  energy model offers an alternative to the  inherent problem of network calibration and confirmation bias for pseudo-labelling. To further support this argument and our extensive experiments,  we provide a set of experiments to showcase the advantages of our energy model vs deep network for pseudo-label generation. We use for CIFAR-10/100 a 13 Layer Network whilst for Mini-Imagenet a ResNet-18. To do this, we run our framework from Section~\ref{section:hybrid} with our energy model, and  without it and allowing the network, directly from $f_\theta$, to generate the pseudo-labels. The results are displayed in Fig.~\ref{fig:ablation1}. In a closer look at the results, we can observe that in Fig.~\ref{fig:ablation1}-a that integrating our energy model encourage better pseudo-labels, which is reflected in having better performance than the deep network. This behaviour is consistently observed across all compared datasets. We also illustrate the certainty of the pseudo-labels over selected epochs from our approach  in Figs.~\ref{fig:ablation1}-a,c.  We observe that our model enforces constant control on the level of certainty of the inferred pseudo-labels over the learning process. This effect can be seen in the plots, where the green shaded area, that reflects the correctness of the pseudo-labels with respect to the ground truth, increases with the evolution of the epochs; whilst the number of incorrect pseudo-labels (see red area) decrease.

\textbf{A Better Energy Model.} Another key motivation of our work is the need for a robust energy model  as existing hybrid techniques~\cite{iscen2019label,sellars2020two,sellars2021laplacenet} have as commonality the use of the energy model of that~\cite{zhou2004learning}. We firstly showed in  Tables~\ref{table::13Layer},~\ref{table::WRN},~\ref{table::MiniImagenet} that our approach outperforms those existing hybrid techniques. The use of such energy model is motivated by its performance as ranked second in Table~\ref{table::energyModels}. To further support our results from that table, we run an additional set of experiments to further evaluate the gain of our energy model vs ~\cite{zhou2004learning} for more complex data -- the ChestX-ray14 (see plot (a)) and CBIS-DDSM (see plot (b)) medical datasets. To do this, we run the hybrid framework from Section~\ref{section:hybrid} with our energy model and the energy of~\cite{zhou2004learning} for different label rates -- that is, changing the green block from Fig.~\ref{fig:hybridFramework}.
The results are displayed in Fig.~\ref{fig:ablation2} in terms of average AUC over the classes. We observe that our technique consistently outperforms that of LGC~\cite{zhou2004learning} for all label rates and both datasets. More precisely, we report a performance improvement in the range of  10\% to 16\% on the different label rates. We also can observe that the both graphical approaches reach a point where more labels are not providing a significant performance improvement. This is an expected behaviour and follows the findings of several early works where the graphical tranductivity bonus is not longer effective as the nature of working on low label rates~\cite{vapnik1998statistical}.

\smallskip
\textbf{Overall Remarks.}  From our results, we now summarise our main highlights over existing techniques:

\faFlash~\textit{ Energy Models for Better Pseudo-labels.} From our experiments, we observe that energy models are a strong approach for pseudo-labelling.  The  intuition behind our technique's performance is that our energy model allows an explicit  control and update of the  predictive uncertainty on the pseudo-labels. By contrast, the compared techniques solely rely on the deep network to get the output without any guarantee or clear understanding on the correctness likelihood of the pseudo-labels.

\faFlash~\textit{Advantages of our Hybrid Model.} Unlike existing energy models, our framework takes advantages of both a robust energy model and deep learning principles. In contrast to pure deep learning techniques, our work offers several mathematical properties such as convergence of the scheme and a better understanding of the technique's behaviour. Finally, in comparison to existing hybrid techniques that use existing energy models and focus on new deep learning mechanism, we are the first work to investigate and propose more robust energy models for hybrid semi-supervised techniques.

\faFlash~\textit{Good Generalisation Capabilities.} In contrast to existing techniques that only present results on natural images, we provided an extensive comparison using natural and medical images. Medical images are more complex and fundamentally different than natural images~\cite{raghu2019transfusion}, and therefore, our results support the good generalisation capability of our technique.
At this point in time, our technique set  a new SOTA for semi-supervised techniques.

\section{Conclusion}
In this work we tackle the problem of classifying with scarce annotations via semi-supervised learning. For this purpose, we proposed a new hybrid framework for semi-supervised classification called {CREPE (1-Lapla\textbf{C}ian g\textbf{R}aph \textbf{E}nergy for \textbf{P}seudo-lab\textbf{E}ls)}. In contrast with existing techniques that focus on developing better mechanisms  for  improving  the  network performance,  we address the problem of how to design better energy models for pseudo-labelling.  The highlight of our work is a novel energy model based on the non-smooth $\ell_1$ norm of the normalised graph 1-Laplacian with thoughtfully selected class priors. Unlike existing deep learning or hybrid techniques, we provided a theoretical analysis of our model. We provide a convergence analysis for our model and its properties. We also show that energy models provide better pseudo-labels than the ones directly obtained from a network. We supported our model by an extensive evaluation using major datasets  composed of natural and medical images. We showed that our technique is able to provide state-of-the-art performance for semi-supervised classification.

\section*{Acknowledgments}
AI Aviles-Rivero gratefully acknowledges support from CMIH and
CCIMI, University of Cambridge. This project has also received funding from the European Union’s Horizon 2020 research and innovation programme under the Marie Skłodowska-Curie grant agreement No 77782. CB Sch\"{o}nlieb acknowledges support from the Leverhulme Trust project on ’Breaking the non-convexity barrier’, the Philip Leverhulme Prize, the Royal Society Wolfson Fellowship, the EPSRC EP/S026045/1, EP/T003553/1 and EP/N014588/1, the Wellcome Innovator Award RG98755,  the CCIMI and the Alan Turing Institute. RT Tan research in this work is supported by MOE2019-T2-1-130.

\ifCLASSOPTIONcaptionsoff
  \newpage
\fi

\bibliographystyle{IEEEtran}
\bibliography{references}

%








\end{document}